\pdfoutput=1
\documentclass{article}
\usepackage[T1]{fontenc}
\usepackage{times}

\usepackage[left=31mm,right=31mm,top=33mm,bottom=20mm]{geometry}
\usepackage{natbib}
\usepackage{booktabs}       
\usepackage{authblk}
\usepackage{amsmath,amssymb,amsfonts}
\usepackage{amsthm}

\usepackage{nicefrac}       
\usepackage{microtype}      
\usepackage{bm}

\usepackage{enumitem}

\usepackage[usenames,dvipsnames]{xcolor}
\usepackage{graphicx}
\usepackage{epsfig}
\usepackage{subcaption}
\usepackage[colorlinks,
            linkcolor=blue,
            citecolor=blue,
            urlcolor=magenta,
            linktocpage,
            plainpages=false]{hyperref}

\bibpunct{(}{)}{;}{a}{,}{,}

\newtheorem{theorem}{Theorem}
\newtheorem{proposition}[theorem]{Proposition}
\newtheorem{lemma}[theorem]{Lemma}

\theoremstyle{definition}
\newtheorem{definition}[theorem]{Definition}

\newtheorem{observation}[theorem]{Observation}
\theoremstyle{remark}

\DeclareMathOperator*{\argmin}{arg\,min}

\newcommand{\RR}{\mathbb{R}}

\newcommand{\actions}{\mathcal{W}}
\newcommand{\regret}{\mathcal{R}}

\newcommand{\groups}{\mathcal{G}}
\newcommand{\Tg}[1]{\mathcal{T}(#1)}

\newcommand{\lossv}{\bm{\ell}}
\newcommand{\loss}{\ell}
\newcommand{\noti}{\bar{i}}
\newcommand{\hedgeloss}{\hat{\loss}}

\title{\bf Best-Case Lower Bounds in Online Learning}

\author[1,2]{Crist\'obal Guzm\'an}
\author[3]{Nishant A. Mehta}
\author[3]{Ali Mortazavi}
\affil[1]{Department of Applied Mathematics, University of Twente}
\affil[2]{IMC, Pontificia Universidad Cat\'olica de Chile}
\affil[3]{Department of Computer Science, University of Victoria}

\date{}

\begin{document}

\maketitle

\begin{abstract}
Much of the work in online learning focuses on the study of sublinear upper bounds on the regret. In this work, we initiate the study of best-case lower bounds in online convex optimization, wherein we bound the largest \emph{improvement} an algorithm can obtain relative to the single best action in hindsight. This problem is motivated by the goal of better understanding the adaptivity of a learning algorithm. Another motivation comes from fairness: it is known that best-case lower bounds are instrumental in obtaining algorithms for decision-theoretic online learning (DTOL) that satisfy a notion of group fairness. Our contributions are a general method to provide best-case lower bounds in Follow The Regularized Leader (FTRL) algorithms with time-varying regularizers, which we use to show that best-case lower bounds are of the same order as existing upper regret bounds: this includes situations with a fixed learning rate, decreasing learning rates, timeless methods, and adaptive gradient methods. In stark contrast, we show that the linearized version of FTRL can attain negative linear regret. Finally, in DTOL with two experts and binary predictions, we fully characterize the best-case sequences, which provides a finer understanding of the best-case lower bounds.
\end{abstract}

	\section{Introduction}
	
	A typical work in online learning would develop algorithms that provably achieve low regret for some family of problems, where low regret means that a learning algorithm's cumulative loss is not much larger than 
that of the best expert (or action) in hindsight. Such a work often focuses on algorithms that exhibit various forms of adaptivity, including \emph{anytime} algorithms, which adapt to an unknown time horizon $T$; \emph{timeless} algorithms, which obtain ``first-order'' regret bounds that replace dependence on the time horizon by the cumulative loss of the best expert; and algorithms like AdaGrad \citep{duchi2011adaptive}, which adapt to the geometry of the data. 
	These examples of adaptivity all involve competing with the best expert in hindsight, but adaptivity comes in many guises. 
	Another form of adaptivity involves upgrading the comparator itself: in the \emph{shifting regret} (also known as the \emph{tracking regret}) \citep{herbster1998tracking}, the learning algorithm competes with 
	the best sequence of experts that shifts, or switches, $k$ times for some small $k$. 
	Naturally, an algorithm with low shifting regret can potentially perform much better than the single best expert in hindsight, thereby obtaining classical regret that is substantially negative.

	Our work serves as a counterpoint to previous works:
	we show that a broad class of learning strategies provably fails, against \emph{any} sequence of data, to substantially outperform the best expert in hindsight. Thus, these strategies are unable to obtain low shifting regret and, more generally, low regret against any comparator sequence that can be substantially better than the best expert in hindsight. More concretely, this paper initiates the study of \emph{best-case lower bounds} in online convex optimization (OCO) for the general family of learning strategies known as Follow The Regularized Leader (FTRL) \citep{abernethy2008competing}. That is, we study the \emph{minimum} possible regret of a learning algorithm over all possible sequences. 
	As we will show, many instances of FTRL --- including adaptive instances that are anytime, timeless, or adapt to gradients like AdaGrad --- never have regret that is much less than the negation of the corresponding regret upper bounds. Thus, while these instances can be adaptive in some ways, they are in a sense prohibited from uniformly obtaining low regret for adaptive notions of regret like the shifting regret. 
	For example, in the setting of decision-theoretic online learning (DTOL) with $d$ experts \citep{freund1997decision}, the well-known anytime version of Hedge (which uses the time-varying learning rate $\eta_t \asymp \sqrt{\log(d)/t}$) enjoys $O(\sqrt{T \log d})$ worst-case regret and, as we show, has $-O(\sqrt{T \log d})$ best-case regret. 
	Moreover, in the same setting under the restriction of two experts and binary losses, we \emph{exactly identify the best-case sequence}, thereby showing that our best-case lower bound for this setting is tight. The structure of this sequence is surprisingly simple, but the arguments we use to pinpoint this sequence are playfully complex, bearing some similarity to the techniques of 
\cite{vanerven2014follow} and \cite{lewi2020thompson}. The latter work \citep{lewi2020thompson} considers the regret of Thompson Sampling in adversarial bit prediction; they use swapping rules and identify best-case sequences, as do we. However, the algorithms and problem settings have important differences.

	A key motivation for our work is a recent result of \cite{blum2018preserving} which shows, in the setting of DTOL, that Hedge with constant learning rate has best-case regret lower bounded by $-O(\sqrt{T})$. 
This result, taken together with worst-case upper bounds of order $O(\sqrt{T})$, is then used to show that if each of finitely many experts approximately satisfies a certain notion of group fairness, then a clever use of the Hedge algorithm (running it separately on each group) also approximately satisfies the same notion of group fairness while still enjoying $O(\sqrt{T})$ regret. However, we stress that their result is very limited in that it applies only to Hedge when run with a known time-horizon. The fixed time horizon assumption also implies that their notion of group fairness also is inherently tied to a fixed time horizon (see Section~\ref{sec:applications} for a detailed discussion), and this latter implication can lead to experts that seem very unfair but which, based on a fixed horizon view of group fairness, are technically considered to be fair. Our best-case lower bounds enable the results of 
\cite{blum2018preserving} 
to hold in much greater generality; in particular, our results enable the use of an anytime version of group fairness, which we feel is truly needed.
	
	To our knowledge, our work is the first to study best-case lower bounds for 
Adaptive FTRL \citep{mcmahan2017survey}, i.e., FTRL with time-varying regularizers that can adapt to the learning algorithms' past observations. 
	The most closely related work 
is a paper by Gofer and Mansour (GM) \citep{gofer2016lower} which, in the setting of online linear optimization (OLO) and when using FTRL with a \emph{fixed} regularizer,\footnote{\cite{gofer2016lower} use learners that follow the gradient of a concave potential, which is essentially equivalent to FTRL.} provides various lower bounds on the regret. 
For instance, they show that for any sequence of data, the regret is nonnegative; we recover this result as a special case of our analysis, and our analysis extends to OCO as well. 
GM also 
lower bound 
what they call the \emph{anytime regret}, which superficially may seem similar to our providing best-case lower bounds for anytime algorithms. Yet, as we explain in Section~\ref{sec:general}, these two notions greatly differ. In short, their analysis lower bounds the maximum regret (over all prefixes of a sequence) for fixed horizon algorithms, whereas our analysis lower bounds the regret for all prefixes (including the minimum) for adaptively regularized algorithms, which includes anytime algorithms.

	A natural question is whether results similar to our results for FTRL also hold for online mirror descent (OMD). In some situations,	such as in OLO when the action space is the probability simplex, the regularizer is the negative Shannon entropy, and the learning rate is constant, our results automatically apply to OMD because the methods are then the same. More generally, it is known that OMD with a time-varying learning rate can fail spectacularly by obtaining linear regret (see Theorem 4 of \cite{orabona2018scale}). Since so much of our work is tied to obtaining anytime guarantees (which would require a time-varying learning rate), we forego providing best-case lower bounds for OMD.
	
	Our main contributions are as follows:
	\begin{enumerate}
		\item We give a general best-case lower bound on the regret for Adaptive FTRL (Section~\ref{sec:general}). Our analysis crucially centers on the notion of adaptively regularized regret, which serves as a potential function to keep track of the regret.
		\item We show that this general bound can easily be applied to yield concrete best-case lower bounds for FTRL with time-varying negative regularizers, one special case being the negative Shannon entropy. 
We also show that an adaptive gradient FTRL algorithm (which can be viewed as a ``non-linearized'' version of the dual averaging version of AdaGrad \citep{duchi2011adaptive}; see Section~\ref{sec:adagrad} for details) admits a best-case lower bound that is essentially the negation of its upper bound (Section~\ref{sec:applications}). 
		\item A widely used variant of FTRL for OCO is to first linearize the losses, leading to linearized FTRL. This method works well with respect to upper bounds, as a basic argument involving convexity goes in the right direction. However, with regards to best-case lower bounds, we show a simple construction (Section~\ref{sec:negative}) for which linearized FTRL obtains $-\Omega(T)$ regret.\footnote{This negative construction, combined with the fact that the dual averaging version of AdaGrad is a linearized version of FTRL, is why we only prove best-case lower bounds for the 
adaptive gradient FTRL algorithm.}
		\item In the setting of DTOL with 2 experts and binary losses, we explicitly identify the best-case sequence, proving that our best-case lower bounds are tight in this setting (Section~\ref{sec:best-case-seq}).
	\end{enumerate}
	
        The next section formalizes the problem setting and FTRL. We then develop the main results.

	\section{Problem Setting and General Prediction Strategies}
	\label{sec:problem}
	
	Before giving the problem setting, we first set some notation. 
	We denote the norm of a vector $w \in \actions$ as $\|w\|$, 
the corresponding dual norm is denoted as $\|\cdot\|_*$, 
and $\log$ is always the natural logarithm.
	
	\paragraph{Problem setting.}
	
	We consider the OCO setting. This is a game between Learner and Nature. In each round $t = 1, 2, \ldots, T$, Learner selects an action $w_t$ belonging to a closed, convex set $\actions \subseteq \RR^d$. Then, with knowledge of $w_1, \ldots, w_t$, Nature responds with a convex loss function $f_t \colon \actions \mapsto \RR$. Learner then observes $f_t$ and suffers loss $f_t(w_t)$. 
	Learner's goal is to minimize its regret, defined as
	\begin{align*}
		\regret_T := \sup_{w \in \actions} \sum_{t=1}^T [ f_t(w_t)-f_t(w) ] ,
	\end{align*}
	which is the gap between Learner's cumulative loss and 
that of the best action in hindsight.
	
	This paper will cover several examples of OCO. 
	The first example is the subclass of OLO problems. 
	In OLO, 
	the loss functions $f_t$ are linear, with $f_t(w) = \langle \ell_t, w \rangle$ for some loss vector $\ell_t \in \RR^d$. A noteworthy special case of OLO is DTOL, also known as the Hedge setting. 
	In DTOL, we take $\actions$ to be equal to the simplex $\Delta_d$ over $d$ outcomes and restrict the loss vectors as $\ell_t \in [0, 1]^d$. 
	We introduce some notation that will be useful in the DTOL setting. 
	For any expert $j \in [d]$, let $L_{t,j} := \sum_{s=1}^t \ell_{s,j}$ denote the cumulative loss of expert $j$ until the end of round $t$. We denote the loss of Learner in round $t$ as $\hat{\ell}_t := \langle \ell_t, w_t \rangle$ and Learner's cumulative loss at the end of round $t$ as $\hat{L}_t := \sum_{s=1}^t \hat{\ell}_s$.

	In this work, we consider the general prediction strategy of FTRL.
	
	\paragraph{FTRL.}
	Let $\Phi_1, \Phi_2, \ldots$ be a possibly data-dependent sequence of regularizers where, for each $t$, the regularizer $\Phi_t$ is a mapping $\Phi_t \colon \actions \rightarrow \RR$ which is allowed to depend on $(f_s)_{s \leq t}$. Then FTRL chooses actions according to the past regularized cumulative loss:
	\begin{align} \label{eqn:adapt_FTRL}
		w_{t+1} = \argmin_{w \in \actions} \Big\{ \sum_{s=1}^t f_t(w)+\Phi_t(w) \Big\}. 
	\end{align}
	We would like to emphasize that this is a very general template. It includes a fixed learning rate regularization, $\Phi_t \equiv \frac{1}{\eta}\Phi$, as well as its variable learning rate counterpart, $\Phi_t = \frac{1}{\eta_t} \Phi$, 
	and arbitrary forms of adaptive choices of $\Phi_t$ based on the past. 
	Despite this adaptivity, in the next section we will show that proving best-case lower bounds for this strategy is quite straightforward.

	The regret attained by FTRL is summarized in the following known result.
	
	\begin{theorem}[Theorem~1 of \cite{mcmahan2017survey}] \label{thm:reg_adapt_FTRL}
		The regret of the FTRL algorithm \eqref{eqn:adapt_FTRL} with sequence of regularizers $(\Phi_t)_{t \geq 1}$ is upper bounded as
		\begin{align*}
			\regret_T \leq \Phi_T(w^*)+\frac12\sum_{t=1}^{T}\|\nabla f_t(w_t)\|_{(t-1),*}^2
		\end{align*}
		Above, $w^* \in \actions$ is the best action in hindsight, and $\|\cdot\|_{(t)}$ is a norm such that $\Phi_t$ is $1$-strongly convex with respect to $\|\cdot\|_{(t)}$.
	\end{theorem}

	\section{A general best-case lower bound}
	\label{sec:general}
	
	We now present our general best-case lower bound for FTRL. 
	Key to our analysis is the concept of \emph{adaptively regularized regret} (hereafter abbreviated as ``regularized regret''), defined as
	\begin{align} \label{eqn:adapt_reg_regret}
		\regret_t^{\Phi_t}= \sup_{w\in\actions} \Big\{ 
		\sum_{s=1}^t[f_s(w^s)-f_s(w)]-\Phi_t(w) 
		\Big\} .
	\end{align}
	The regularized regret can easily be related to the regret as
	\begin{align}
		\regret_t^{\Phi_t} \leq \regret_t-\inf_{w \in \actions} \Phi(w). \label{eqn:regret_vs_reg_regret} 
	\end{align}
	
	Also, applying \eqref{eqn:adapt_FTRL}, the following re-expression of the regularized regret is immediate:
	\begin{align}
		\regret_t^{\Phi_t} &= \sum_{s=1}^t [f_s(w_s) - f_s(w_{t+1})] - \Phi_t(w_{t+1}) . \label{eqn:reg_regret_eval}
	\end{align}
	
	\begin{theorem}[Best-case lower bound on regret for adaptive FTRL] \label{thm:BCLB}
		Consider the setting of online convex optimization, and the adaptive FTRL strategy \eqref{eqn:adapt_FTRL}. Suppose that there exists a sequence $(\alpha_t)_{t \in [T]}$ such that 
		$\Phi_t(w_t) \leq \Phi_{t-1}(w_t)+\alpha_t$ 
		for all $t \in [T]$. 
		Then
		\begin{align*}
			\regret_T \geq \inf_{w \in \actions} \Phi_T(w) 
			- \inf_{w \in \actions} \Phi_0(w)
			-\sum_{t=1}^T \alpha_t .  
		\end{align*}
	\end{theorem}
	
	\begin{proof}
		We start by inductively bounding the adaptively regularized regret: 
		\begin{align*}
			\regret_{t+1}^{\Phi_{t+1}} - \regret_t^{\Phi_t} 
			&= \max_{w \in \actions} \Big\{ \sum_{s=1}^{t+1} [ f_s(w_s) - f_s(w) ] - \Phi_{t+1}(w) \Big\}
			- \sum_{s=1}^{t+1} [ f_s(w_s) - f_s(w_{t+1}) ] + \Phi_{t}(w^{t+1}) \\
			&\geq  - \Phi_{t+1}(w_{t+1}) + \Phi_{t}(w_{t+1}) \\
			&\geq -\alpha_{t+1} ,
		\end{align*}
		where in the first equality we used \eqref{eqn:reg_regret_eval}. We conclude that $\regret_T^{\Phi_T} \geq \regret_0^{\Phi_0} - \sum_{t=1}^T \alpha_t$. 
		Next, we use \eqref{eqn:regret_vs_reg_regret} to conclude that
		\begin{align*}
			\regret_T &\geq  \regret_T^{\Phi_T} + \inf_{w \in \actions} \Phi_T(w) 
			\geq \inf_{w \in \actions} \Phi_T(w) + \regret_0^{\Phi_0} - \sum_{t=1}^T \alpha_t \\
			&= \inf_{w \in \actions} \Phi_T(w) - \inf_{w \in \actions} \Phi_0(w) - \sum_{t=1}^T \alpha_t ,
		\end{align*}
		where in the last equality we used that 
		$\regret_0^{\Phi_0} = \sup_{w \in \actions} \{ -\Phi_0(w) \}$.
	\end{proof}

The closest results to Theorem~\ref{thm:BCLB} of which we are aware are contained in the intriguing work of \cite{gofer2016lower}, who provided lower bounds for FTRL with a fixed regularizer in the setting of OLO. Their Theorem 1 shows that the best-case regret is nonnegative. They also lower bound a notion they call the anytime regret (see Theorem 5 of \cite{gofer2016lower}). The anytime regret for a sequence as they define it is actually the maximum regret over all prefixes of the sequence (where the sequence has a fixed time horizon); ultimately, the lower bound they obtain depends on the quadratic variation of the sequence as computed on a \emph{fixed} time horizon. 
Related to this, Gofer and Mansour's analysis is for algorithms that use the same regularizer in all rounds. They lament that it is unclear how to extend their analysis to handle time-varying learning rates\footnote{A time-varying learning rate gives rise to perhaps the most basic form of an adaptive regularizer.}. Our goal is rather different, as taking the maximum regret over all prefixes says nothing about how large (in the negative direction) the regret could be for some particular prefix. Thus, our style of analysis, which provides a lower bound on the regret for all time horizons (which, in particular, provides a lower bound on the \emph{minimum} over all prefixes), differs greatly from theirs and is what is needed. We think the different styles of our work and theirs stems from the regret being nonnegative in their paper, whereas it need not be for time-varying regularizers.

Theorem~\ref{thm:BCLB} possesses considerable generality, owing to its applying to FTRL with adaptive regularizers. We highlight just a few applications in the next section.
	
\section{Best-case lower bounds in particular settings}
\label{sec:applications}
	
Theorem~\ref{thm:BCLB} presented in last section, despite its simplicity, is an extremely powerful method, capable of addressing several of the settings where FTRL attains sublinear regret. Next we proceed to enumerate some important examples of instances of FTRL, comparing existing worst-case upper bounds on the regret with our best-case lower bounds.  
	
  \subsection{Non-increasing learning rates, timeless algorithms, fairness}
	\label{sec:anytime-fairness}
	
	We first present several examples in which the elements of the sequence $(\Phi_t)_{t \geq 1}$ take the form $\Phi_t = \frac{1}{\eta_t} \Phi$ for a fixed regularizer $\Phi$ and a time-varying learning rate $\eta_t$.
	
	\paragraph{Constant learning rate.}
	The simplest example is that of a fixed learning rate $\eta_t \equiv \eta$, 
	which means 
	$\Phi_t(w) = \frac{1}{\eta} \Phi(w)$. 
	Taking $\alpha_t = 0$ for all $t$, Theorem~\ref{thm:BCLB} immediately implies that the regret is always nonnegative; this implication was previously shown by Gofer and Mansour (see Thm.~1 of \cite{gofer2016lower}). A simple consequence is that the Follow the Leader (FTL) strategy (i.e., $\Phi_t \equiv 0$) has nonnegative regret, which also can be inferred from the results of \cite{gofer2016lower}. Although we cannot find a precise reference, we believe that it was already known that FTL always obtains nonnegative regret (even prior to \citep{gofer2016lower}).

	\paragraph{Time-varying learning rate.}
	More generally, taking a time-varying learning rate, Theorem~\ref{thm:BCLB} gives
	\begin{align} \label{eqn:time-varying-eta}
		\regret_T 
		\geq \big( \frac{1}{\eta_T} - \frac{1}{\eta_0} \big) \inf_{w \in \actions} \Phi(w) 
		- \sum_{t=1}^T \big( \frac{1}{\eta_t} - \frac{1}{\eta_{t-1}} \big) \Phi(w_{t+1} ) .
	\end{align}
	A typical strategy to obtain sublinear anytime worst-case regret is to set the learning rate as $\eta_t = \eta / \sqrt{t+1}$ for some constant $\eta$ that depends on various known problem-dependent constants. Continuing from \eqref{eqn:time-varying-eta} with this setting further implies that
	\begin{align*} 
		\regret_T \geq \frac{1}{\eta} (\sqrt{T+1} - 1) \inf_{w \in \actions} \Phi(w)  
		- \frac{1}{\eta} \sup_{w \in \actions} \Phi(w) \sum_{t=1}^T \frac{1}{2\sqrt{t+1}} .
	\end{align*}
	If $\Phi$ is nonpositive --- as holds when $\actions$ is the $d$-dimensional simplex and $\Phi$ is the negative Shannon entropy --- the above is further lower bounded by
	\begin{align} \label{eqn:pre-decreasing-hedge} 
		\frac{1}{\eta} (\sqrt{T+1} - 1) \inf_{w \in \actions} \Phi(w)  
		- \frac{1}{\eta} \sup_{w \in \actions} \Phi(w) \sqrt{T+1} .
	\end{align}
	
	A particularly interesting example is the DTOL setting. In this setting, when we take $\Phi$ to be the negative Shannon entropy $\Phi(w) = \sum_{j=1}^d w_j \log w_j$ and set $\eta = 2 \sqrt{\log d}$ so that $\eta_t = \sqrt{(\log d)/(t+1)}$ recovers the standard anytime version of Hedge which has also been called Decreasing Hedge \citep{mourtada2019optimality}. In round $t$, this algorithm plays $w_t$ such that $w_{t,j} \propto \exp(-\eta_{t-1} L_{t-1,j})$ for $j \in [d]$, and we have the following anytime best-case lower bound and worst-case upper bound on the regret:
	\begin{align*}
		-\frac{1}{2} \sqrt{T \log d} \leq \regret_T \leq \sqrt{T \log d} .
	\end{align*}
	The lower bound holds from \eqref{eqn:pre-decreasing-hedge} combined with $\sqrt{T + 1} - 1 \leq \sqrt{T}$ and $-\log d \leq \Phi \leq 0$, while the upper bound is from Theorem 2 of \cite{chernov2010prediction}. The upper bound is minimax optimal in terms of the rate and, asymptotically (letting both $d$ and $T$ go to infinity) has a constant that is optimal up to a factor of $\sqrt{2}$. As we show in Section~\ref{sec:best-case-seq}, in the case of $d = 2$ the lower bound also has the optimal rate.

	\paragraph{Timeless algorithms.}
	
	In the context of DTOL, it is straightforward to adapt our analysis for Decreasing Hedge to Hedge with any non-increasing learning rate. One example of interest is
	\begin{align*} 
		\eta_t = -\log \left( 1 - \min \left\{ \frac{1}{4}, \sqrt{\frac{2 \log d}{L^*_t}} \right\} \right) ,
	\end{align*}
	where $L^*_T = \min_{j \in [d]} L_{T,j}$ is the cumulative loss of the best expert. 
	This choice of adaptive learning rate yields an anytime upper bound on the regret of
	$O \left( \sqrt{L^*_T \log d} + \log d \right)$ in the DTOL setting (see Theorem 2.1 of \cite{auer2002adaptive}, who actually prove this result in the more general setting of prediction with expert advice). Such a bound is called timeless because rounds in which all experts suffer the same loss have no effect on the bound \cite{derooij2014follow}. This is a natural property to have in this setting, and with the above choice of learning rate, we have the following timeless\footnote{Technically, the upper and lower bounds as stated are not timeless, but it is easy to see that any round for which all experts have the same loss can be replaced by a round where all experts have zero loss, with no change in the algorithm's behavior nor its regret. Our regret bounds on the modified loss sequence then become timeless.}
 best-case lower bound:
	\begin{align}  \label{eqn:timeless-dtol-bclb}
		\regret_T \geq \min \left\{ 0, -\sqrt{\frac{L^*_T \log d}{2}} + 4 \log d \right\} ;
	\end{align}
	a brief derivation is in Appendix~\ref{app:timeless}. Again, notice the similarity between the upper and lower bounds.

	\subsection{Group fairness in online learning}

In this section, we show how our anytime best-case lower bounds can be used to achieve a certain notion of group fairness in online learning, extending previous results of \cite{blum2018preserving}. We begin with an overview of our results and then provide more detailed derivations.

\subsubsection{Overview of results}
	
	In a pioneering work, Blum, Gunasekar, Lykouris, and Srebro (BGLS) \citep{blum2018preserving} considered a notion of group fairness in DTOL. In their setup, the DTOL protocol is augmented so that, at the start of each round $t$, Nature selects and reveals to Learner a group $g_t$ belonging to a set of groups $\groups$ prior to Learner's playing its action $w_t$. They assume that for a known, fixed time horizon $T$, each expert $j \in [d]$ has balanced mistakes across groups in the following sense: 
\begin{quote}
	For any $g \in \groups$, let $\Tg{g} := \{t \in [T] \colon g_t = g\}$ denote the rounds belonging to group $g$, 
	and let $L_{\Tg{g}, j} := \sum_{t \in \Tg{g}} \ell_{t,j}$ denote the cumulative loss of expert $j \in [d]$ when considering only the rounds in $\Tg{g}$; then we say that expert $j$ is fair in isolation if
	\begin{align*}
		\frac{L_{\Tg{g},j}}{|\Tg{g}|} = \frac{L_{\Tg{g'},j}}{|\Tg{g'}|} 
		\quad \text{for all } g, g' \in \groups .
	\end{align*}
\end{quote}
	By leveraging a best-case lower bound for Hedge\footnote{They actually use the well-known multiplicative weights method, but the difference from Hedge is minor.} run with a constant learning rate with a known time horizon (a natural choice is $\eta \asymp \sqrt{(\log d)/T}$), they show that the following strategy also satisfies group fairness: run a separate copy of Hedge for each group, so that for any group $g$, the copy corresponding to group $g$ is run on the subsequence corresponding to the rounds $\Tg{g}$. For brevity, we call this ``Interleaved Hedge''. Then on the one hand, the regret of Interleaved Hedge satisfies
	\begin{align*}
		\regret_T = O(\sqrt{|\groups| T \log d}) .
	\end{align*}
	In addition, by virtue of BGLS's $-O(\sqrt{T \log d})$ lower bound for Hedge when fed $T$ rounds, their Interleaved Hedge enjoys the following group fairness guarantee:
	\begin{align*}
		\hat{L}_{\Tg{g}} - \hat{L}_{\Tg{g'}} 
		= O \left( \sqrt{\frac{\log d}{T_0}} \right) 
		\quad \text{for all } g, g' \in \groups \text{ and } T_0 := \min_g |\Tg{g}| ,
	\end{align*}
	where we adopt the notation $\hat{L}_{\Tg{g}} := \sum_{t \in \Tg{g}} \hat{\ell}_t$. 
	Several remarks are in order. First, using our improved nonnegative best-case lower bound for Hedge with a constant learning rate (which again, is not a new result), their group fairness guarantee can be improved to
	\begin{align*}
		\hat{L}_{\Tg{g}} - \hat{L}_{\Tg{g'}} 
		\leq \sqrt{\frac{\log d}{T_0}} 
		\quad \text{for all } g, g' \in \groups \text{ and } T_0 := \min_g |\Tg{g}| .
	\end{align*}
	Second, in order to deploy Hedge instances for each group with the correct constant learning rates, their algorithm needs to know $|\Tg{g}|$ for each group $g \in \groups$, at least within a reasonable constant factor. This is a far stronger assumption than the already strong assumption of a known time horizon.
	
	Using our anytime best-case lower bound for Decreasing Hedge, combined with the analysis of BGLS, it is straightforward to vastly extend their results in each of the following ways:
	\begin{enumerate}
		\item Using the same fixed horizon notion of fairness as in their paper, using a copy of Decreasing Hedge for each group's instance, when can avoid needing to assume that each group's cardinality $|\Tg{g}|$ is known. We give a sketch of how to modify their analysis in Section~\ref{sec:fairness}.
		\item The most interesting extension, whose possibility was the original basis of our entire work, is that we can now upgrade BGLS's notion of group fairness to its anytime sibling. This involves measuring, for every prefix of the length-$T$ game, the discrepancy between the error rates of any pair of groups. This is an arguably more natural notion of fairness, as it avoids situations where an expert purports to be fair while having all of its mistakes for one group occur in the first half of the game. Since we now have an anytime best-case lower bound for Decreasing Hedge, we have the requisite piece needed to show that Interleaved (Decreasing) Hedge satisfies the same notion of anytime group fairness. Our timeless best-case lower bounds also apply here, giving that extension as well. All details can be found in Section~\ref{sec:fairness}.
	\end{enumerate}

Finally, whereas a key message of \cite{blum2018preserving} is that adaptive algorithms (in the sense of having low shifting regret) cannot satisfy group fairness, at least when using the interleaved approach, our best-case lower bounds do cover many other types of adaptivity. This shows that with regards to group fairness and the interleaved strategy, the group-fair adversarial online learning tent is actually quite large.

\subsubsection{Detailed coverage of results}
\label{sec:fairness}
	
In their work, BGLS \citep{blum2018preserving} express regret bounds in a somewhat different language than we do. Rather than dealing directly with the regret, they instead analyze the \emph{approximate regret}. The $\varepsilon$-approximate regret relative to expert $j \in [d]$ is defined as
\begin{align} \label{eqn:approx-regret}
\hat{L}_T - (1 + \varepsilon) L_{T,j} ,
\end{align}
where $L_{T,j} := \sum_{t=1}^T \ell_{t,j}$ is the cumulative loss of expert $j$.

In order to more easily compare to their results, and to interpret their results in the large body of literature that focuses on the actual (non-approximate) regret, we first show how to modify some of their analysis to give results based on the actual regret (hereafter simply referred to as the ``regret'').

First, recall that BGLS use the multiplicative weights algorithm, which sets $w_t$ as $w_{t,j} \propto (1 - \tilde{\eta})^{L_{t-1,j}}$ for $j \in [d]$ for a learning rate parameter $\tilde{\eta}$. 
In BGLS's proof of their Theorem 3, they first give worst-case regret upper and lower bounds for Hedge (``multiplicative weights'' in their work). Specifically, they show that
\begin{align*}
(1 - 4 \tilde{\eta}) \cdot L^*_T 
\leq \hat{L}_T 
\leq (1 + \tilde{\eta}) L^*_T + \frac{\log d}{\tilde{\eta}} .
\end{align*}

An optimal, non-anytime worst-case tuning of $\tilde{\eta}$ then yields matching-magnitude regret lower and upper bounds of $-O(\sqrt{T \log d})$ and $O(\sqrt{T \log d})$ respectively. We note in passing that the anytime version of multiplicative weights (which uses incremental updates) is equivalent to OMD with a time-varying learning rate and is known to achieve linear regret \cite[Theorem 4]{orabona2018scale}.

We can obtain a similar bound for constant learning rate Hedge with an optimal, non-anytime worst case tuning of $\eta$ (defined as in the main text of our paper), with the improvement\footnote{Since this is a fixed horizon setting for now, the constant in the upper bound can be improved, but for simplicity we will just use the constant of 1.}
\begin{align*}
0 \leq \hat{L}_T \leq \sqrt{T \log d} .
\end{align*}

Next, we briefly explain how to modify BGLS's proof of their Theorem 3. We first need some notation (largely derived from BGLS, but with small modifications to integrate more nicely into our notation). For any group $g$, let $j^*(g)$ be the best expert when considering the rounds involving group $g$ (i.e., $\Tg{g}$). Therefore, $j^*(g) \in \argmin_{j \in [d]} L_{\Tg{g},j}$. Next, recall that Interleaving Hedge runs a separate copy of Hedge for each group. Let $g^*$ be the group whose copy of Hedge obtains the lowest average loss. That is, $g^*$ is such that
\begin{align*}
\frac{\hat{L}_{\Tg{g^*}}}{|\Tg{g^*}|} = \min_{g \in \groups} \frac{\hat{L}_{\Tg{g}}}{|\Tg{g}|} .
\end{align*}
We now pick up at the last math display in Appendix C of BGLS \citep{blum2018preserving}; this is the step where the lower and upper regret bounds are used. Adjusting their analysis using our bounds, we have for a fixed time horizon $T$ and when the copy of Hedge running on group $g$ uses learning rate\footnote{Yes, it is a very strong assumption to assume that the $\Tg{g}$'s are known ahead of time, but this is precisely our point. We will relax this soon.} $\eta^{(g)} = \sqrt{\frac{\log d}{|\Tg{g}|}}$:
\begin{align*}
\frac{\hat{L}_{\Tg{g}}}{|\Tg{g}|} - \frac{\hat{L}_{\Tg{g^*}}}{|\Tg{g^*}|} 
&\leq \frac{L_{\Tg{g},j^*(g)}}{|\Tg{g}|} + \sqrt{\frac{\log d}{|\Tg{g}|}} 
        - \frac{L_{\Tg{g^*},j^*(g^*)}}{|\Tg{g^*}|} \\
&\leq \frac{L_{\Tg{g},j^*(g^*)}}{|\Tg{g}|} - \frac{L_{\Tg{g^*},j^*(g^*)}}{|\Tg{g^*}|} 
          + \sqrt{\frac{\log d}{|\Tg{g}|}} \\
&= \sqrt{\frac{\log d}{|\Tg{g}|}} ,
\end{align*}
where the first inequality uses the regret lower and upper bounds, the second inequality is based on the optimality of $j^*(g)$ for group $g$, and the equality uses fairness in isolation (for any pair of groups $g, g' \in \groups$ and expert $j$ (including $j^*(g^*)$), we have that $\frac{L_{\Tg{g},j}}{|\Tg{g}|} = \frac{L_{\Tg{g'},j}}{|\Tg{g'}|}$.

\paragraph{Decreasing Hedge (Anytime analysis).}
Suppose now that Interleaving Hedge uses copies of Decreasing Hedge with time-varying learning rate $\eta_t = 2 \sqrt{(\log d)/(t+1)}$. Note that in this case, the copy for group $g$ increments its internal round only each time a new round for group $g$ appears. We can then automatically apply our anytime lower bound on the regret of Decreasing Hedge (together with the already well-known regret upper bound)  to obtain
\begin{align*}
\frac{\hat{L}_{\Tg{g}}}{|\Tg{g}|} - \frac{\hat{L}_{\Tg{g^*}}}{|\Tg{g^*}|} 
&\leq \frac{L_{\Tg{g},j^*(g)}}{|\Tg{g}|} + \sqrt{\frac{\log d}{|\Tg{g}|}} 
        - \frac{L_{\Tg{g^*},j^*(g^*)}}{|\Tg{g^*}|} + \frac{1}{2} \sqrt{\frac{\log d}{|\Tg{g^*}|}} \\
&\leq \frac{L_{\Tg{g},j^*(g^*)}}{|\Tg{g}|} - \frac{L_{\Tg{g^*},j^*(g^*)}}{|\Tg{g^*}|} 
          + \sqrt{\frac{\log d}{|\Tg{g}|}} + \frac{1}{2} \sqrt{\frac{\log d}{|\Tg{g^*}|}} \\
&= \sqrt{\frac{\log d}{|\Tg{g}|}} + \frac{1}{2} \sqrt{\frac{\log d}{|\Tg{g^*}|}} .
\end{align*}
Therefore, if in hindsight we have $T_0 := \min_{g \in \groups} |\Tg{g}|$, then the following fairness guarantee holds:
\begin{align*}
\frac{\hat{L}_{\Tg{g}}}{|\Tg{g}|} - \frac{\hat{L}_{\Tg{g^*}}}{|\Tg{g^*}|} 
\leq \frac{3}{2} \sqrt{\frac{\log d}{T_0}} .
\end{align*}
Note that even though Learner need not know the time horizon $T$ nor $|\Tg{g}|$ for each $g$, the above guarantee still can only hold for a fixed time horizon $T$. This is because the definition of fairness in isolation only holds for a fixed time horizon. By adjusting this definition (which we believe is sensible), we can go further. We now briefly expore this extension.

\paragraph{Anytime fairness.}
As mentioned in Section~\ref{sec:anytime-fairness}, the fixed horizon view of fairness in isolation can be very limiting. Rather than opting for fairness in isolation to hold for a fixed time horizon, we argue that it is more natural for this definition to hold in the following anytime sense:
\begin{align*}
  \frac{L_{\Tg{g},j}}{|\Tg{g}|} = \frac{L_{\Tg{g'},j}}{|\Tg{g'}|} 
  \quad \text{for all } g, g' \in \groups \text{ and for all } T ;
\end{align*}
again, note that the time horizon $T$ is implicit in each $\Tg{g}$.

The attentive reader may notice that this definition can be overly restrictive on Nature (perhaps impossibly so), and therefore it is natural to allow the equality to hold only approximately. We return to this point at the end of this section.

Using the above anytime version of fairness in isolation, it is straightforward to extend the previous result to the following new result. Just like above, suppose now that Interleaving Hedge uses copies of Decreasing Hedge with time-varying learning rate. Then, for all time horizons $T$,\footnote{Note that in the below, $g^*$ and each $\Tg{g}$ implicitly depend on the time horizon $T$.} 
\begin{align*}
\frac{\hat{L}_{\Tg{g}}}{|\Tg{g}|} - \frac{\hat{L}_{\Tg{g^*}}}{|\Tg{g^*}|} 
\leq \frac{3}{2} \sqrt{\frac{\log d}{T_0}} .
\end{align*}
where we recall that $T_0 = \min_{g \in \groups} |\Tg{g}|$.

\paragraph{Approximate fairness.}
Just as in BGLS's work, it is possible to extend our results to cases where fairness holds only approximately. Such an extension is certainly warranted in the case of anytime fairness, as mentioned above. This extension requires only straightforward modifications to the above analysis and so we do not give further details here. In the case of anytime fairness, it further makes sense for approximate fairness to be defined according to a rate. We leave this extension to a future paper.

	\subsection{Adaptive gradient FTRL}
	\label{sec:adagrad}
	
	Inspired by the \emph{adaptive gradient} (AdaGrad) algorithm for OCO \citep{duchi2011adaptive}, we consider an adaptive gradient FTRL algorithm, where we use quadratic regularizers, $\Phi_t(w)=\frac{1}{2\eta}\langle w, H_t w\rangle$. We include its two variants. In the below, we use the notation 
	$g_t := \nabla f_t(w_t)$ and $\delta>0$ is a fixed number:
\begin{enumerate}[label=(\alph*)]
\item Diagonal: $H_t = \delta I+ \mbox{Diag}(s_t)$, 
          where $s_t=\Big(\big(\sqrt{\sum_{\tau=1}^tg_{j,t}^2}\big)_{j\in[d]}\Big)$,
\item Full-matrix: $H_t = \delta I+G_t^{1/2}$, 
          where $G_t=\sum_{\tau=1}^t g_tg_t^{\top}$.
	\end{enumerate}
	
	We emphasize that the proposed algorithm does not exactly match AdaGrad from \citep{duchi2011adaptive}. To resolve this inconsistency, we use the regret upper bounds 
	from Theorem~\ref{thm:reg_adapt_FTRL}, combined with upper bounds on the gradient norms
	that appear in such upper bounds, proved in \citep{duchi2011adaptive}. This leads to the following
	result.
	\begin{theorem}[From \citep{duchi2011adaptive}] \label{thm:Adagrad_UB}
		Consider the setting of OCO. Given $1\leq p\leq \infty$, we denote by $D_p$ the diameter of ${\cal W}$ in the $\|\cdot\|_p$ norm, and $M_p$ is a bound on the $\|\cdot\|_p$ norm of the gradients for any possible loss. Then the adaptive gradient FTRL algorithm satisfies the following bounds:\\
		\begin{enumerate}[label=(\alph*)]
			\item 
			Diagonal: If $\delta=M_{\infty}$, then 
			$\sum_{t=1}^T\|\nabla f_t(w_t)\|_{(t-1),\ast}^2 
			\leq 2 \eta \sum_{j=1}^d \sqrt{\sum_{t=1}^T g_{t,j}^2}$. 
			In particular, setting $\eta = D_{\infty}$, we obtain an upper bound on the regret $\frac{D_2^2}{D_{\infty}}M_{\infty}+D_{\infty}\sum_{j=1}^d\sqrt{\sum_{t=1}^Tg_{t,j}^2}$.
			\item 
			Full matrix: If $\delta = M_2$, then
			$\sum_{t=1}^T \|\nabla f_t(w_t)\|_{(t-1),\ast}^2 \leq 2 \eta \, \mbox{Tr}(G_T^{1/2})$. 
			In particular, setting $\eta = D_2$, 
			we obtain an  upper bound on the regret  $D_2[M_2/2+\mbox{Tr}(G_T^{1/2})]$. 
		\end{enumerate}
	\end{theorem}
	
	We proceed to best-case lower bounds. 
	The proof strategy follows similar telescopic recursions to those used in \citep{duchi2011adaptive,mcmahan2017survey}.
	\begin{proposition} \label{prop:BCLB_adaptive_FTRL}
	In the OCO setting, the adaptive gradient FTRL strategy with parameter tuning as in Theorem \ref{thm:Adagrad_UB}, attains
	best-case lower bounds on the regret of $-(D_{\infty}/2)\sum_{j=1}^d\sqrt{\sum_{t\in[T]}g_{t,j}^2}$, in the diagonal case; and $-(D_2/2)\mbox{Tr}(G_T^{1/2})$, in the full-matrix case.
	\end{proposition}
	
	Notice that these bounds match closely the respective worst-case upper bounds.

	\begin{proof} We start noticing that without loss of generality, we may assume that $0\in{\cal W}$ (if not, shift the regularizers by centering them at any fixed point $\bar w\in {\cal W}$. This implies that $\inf_{w\in{\cal W}}\Phi_t(w)=0$, so we only need to focus on the coefficients $\alpha_t$. 

	\begin{enumerate}
	\item \,\!{\bf Diagonal case.}
	First, for the diagonal case, let $s_t$ be the $d$-dimensional vector with coefficients $s_{t,j} = \sqrt{\sum_{s=1}^t g_{t,j}^2}$. Then, if we let $\mathbf{1}\in\mathbb{R}^d$ be the all-ones
	vector,
	\begin{multline*}
		\alpha_{t+1} = \langle w_{t+1}, (H_{t+1} - H_{t}) w_{t+1} \rangle 
		= \langle w_{t+1}, \mbox{Diag}(s_{t+1} - s_{t}) w_{t+1} \rangle \\
		\leq \max_{j \in [d]} (w_j^{t+1})^2 \|s_{t+1} - s_{t}\|_1
		= \max_{j \in [d]} (w_j^{t+1})^2 \langle s_{t+1} - s_{t}, \mathbf{1} \rangle ,
	\end{multline*}
	where we used that $s_t$ is coordinate-wise nondecreasing. Next, 
	\begin{align*}
		\sum_{t=1}^T \alpha_t 
		= \frac{1}{2 \eta} \sum_{t=0}^{T-1} \langle w_{t+1}, (H_{t+1} - H_{t}) w_{t+1} \rangle 
		\leq \frac{1}{2 \eta} \max_{t\in [T]} \|w_t \|_{\infty}^2 \langle s_T - s_0 , \mathbf{1} \rangle .
	\end{align*}
	Hence, choosing $\eta = D_{\infty}$, we obtain a lower bound on regret
	$-\frac12D_{\infty} \sum_{j=1}^d \sqrt{\sum_{t=1}^T g_{t,j}^2}$.
	
	\item \,\!{\bf Full-matrix case.}
	Now for the full-matrix update, we can proceed similarly. First,
	\begin{align*}
		2\eta\alpha_{t+1} &= \langle w_{t+1}, ( H_{t+1} - H_t ) w_{t+1} \rangle
		= \langle w_{t+1}, (G_{t+1}^{1/2} - G_t^{1/2}) w_{t+1} \rangle \\
		&\leq \|w_{t+1}\|_2^2 \, \lambda_{\mbox{\tiny max}}(G_{t+1}^{1/2} - G_t^{1/2})
		= \|w_{t+1}\|_2^2 \, \mbox{Tr}(G_{t+1}^{1/2} - G_t^{1/2}) ,
	\end{align*}
	where we used that the matrix $G_{t+1}^{1/2} - G_t^{1/2}$ is positive semidefinite. Now, summing over $t$,
	\begin{align*}
		\sum_{t=1}^T \alpha_t 
		= \frac{1}{2 \eta} \sum_{t=1}^T \langle w_{t+1}, (H_{t+1} - H_t ) w_{t+1} \rangle 
		\leq \frac{1}{2 \eta} \max_{t \in [T]} \|w_t\|_2^2 
		\left( \mbox{Tr}(G_{T}^{1/2}) - \mbox{Tr}(G_0^{1/2})  \right) .
	\end{align*}
	To conclude, for $\eta = D_2$, the regret is lower bounded by
	\begin{align*}
		\regret_T\geq - \frac{D_2}{2}\, \mbox{Tr}(G_{T}^{1/2}) ,
	\end{align*}
	which proves the result.
	\end{enumerate}
	\end{proof}

	\section{Negative results}
	\label{sec:negative}
	
	So far, our study of best-case lower bounds has focused on adaptive FTRL algorithms. 
	A major drawback of FTRL is that the iteration cost of each subproblem grows linearly with the number of rounds. It is then tempting to consider more efficient algorithms attaining comparable regret upper bounds. We discuss such possibilities for two natural algorithms: linearized FTRL and OMD.
	
\subsection{Linearized FTRL}
\label{sec:linearized-ftrl}

In this section, we give a simple construction in which linearized FTRL obtains negative linear regret, i.e., regret which is $-\Omega(T)$. 
The problem is a forecasting game with binary outcomes, the action space $\mathcal{D}$ equal to the 2-simplex $[0, 1]$, and the squared loss $\ell(p, y) = \frac{1}{2} (p - y)^2$. This can be cast as an OCO problem by taking $\actions = [0, 1]$ and, for each $t \in [T]$, setting $f_t(p) = (p - y_t)^2$ for some outcome $y_t \in \{0, 1\}$ selected by Nature.

We consider linearized FTRL using the negative Shannon entropy regularizer; this is equivalent to exponentiated gradient descent \citep{kivinen1997exponentiated} and uses the update 
$p_t = \frac{1}{1 + \exp \left( \eta_t G_{t-1} \right)}$,
where $G_{t-1} = \sum_{s=1}^{t-1} g_s$ and each $g_s = p_s - y_s$ is the gradient of the loss with respect to $p_s$ under outcome $y_s$. 
In this situation, an $O(\sqrt{T})$ anytime regret upper bound can be obtained by employing the time-varying learning rate $\eta_t = 1/\sqrt{t}$ (see Theorem 2.3 of \cite{cesabianchi2006prediction} or Theorem 2 of \cite{chernov2010prediction}).
	
Let Nature's outcome sequence be the piecewise constant sequence consisting of $T/2$ zeros followed by $T/2$ ones. Therefore, the best action in hindsight is $p^* = 0.5$. 
We now state our negative result.
\begin{theorem}
  \label{thm:linearized-ftrl}
  In the construction above, linearized FTRL obtains regret $\regret_T = -\Omega(T)$.
\end{theorem}
\begin{proof}
Let $q_0$ and $q_1$ be constants satisfying $0 < q_0 < 1/2 < q_1 < 1$; we will tune these constants later.
For the analysis, we divide the $T$ rounds into 4 segments:
\begin{enumerate}[label=(\roman*)]
\item the rounds in the first half for which $p_t \geq q_0$; 
\item the remaining rounds in the first half;
\item the rounds in the second half for which $p_t \leq q_1$; 
\item the remaining rounds in the second half.
\end{enumerate}

The basic idea of the proof is to show that $p_t < q_0$ after a constant number of rounds (where the constant depends on $q_0$) and hence the first segment is of constant length. In the second segment, the algorithm picks up negative linear regret since $p_t < q_0 < 1/2 = p^*$ and $y_t = 0$ in the second segment. In the third segment, where the outcomes now satisfy $y_t = 1$, we show that the algorithm can take at most a linear number of rounds (but with a suitably small coefficient) before satisfying $p_t > q_1$ and so picks up at most linear positive regret (but again, with a suitably small coefficient). 
Finally, in the last segment, the algorithm once again picks up negative linear regret since in this segment $y_t = 1$ and $p_t > q_1 > 1/2 = p^*$ and also, as we show, this segment is linear in length.
  
In more detail, we will upper bound the regret via the following claims (proved in Appendix~\ref{app:linearized-ftrl}).
  
\paragraph{Claim 1:} 
In the first half of the game, we always have $p_t \leq \frac{1}{2}$.

A simple consequence of Claim 1 is that the first segment's contribution to the regret is nonpositive (recall that $y_t = 0$ in the first half).

\paragraph{Claim 2:} 
The number of rounds in the first segment is at most $t_1 = O(1)$.

\paragraph{Claim 3:} 
The number of rounds in the third segment is at most 
$t_3 = \frac{q_0}{2 (1 - q_1)} T + O \bigl( \sqrt{T} \bigr)$.

Putting the 3 claims together and recalling that the contribution to the regret is nonpositive in the first segment, we have that the regret of the algorithm is at most:
\begin{align} \label{eqn:from-claims}
-\underbrace{ \left( \frac{T}{2} - t_1 \right) \cdot \left( \frac{1}{8} - \frac{q_0^2}{2} \right) }_{\text{second segment}} 
+ \underbrace{ t_3 \cdot \left( \frac{1}{2} - \frac{1}{8} \right) }_{\text{third segment}} 
- \underbrace{ \left( \frac{T}{2} - t_3 \right) \cdot \left( \frac{1}{8} - \frac{(1 - q_1)^2}{2} \right) }_{\text{fourth segment}} .
\end{align}
  
Substituting the values of $t_1$ and $t_3$, grouping terms, and rearranging, the above is equal to
\begin{align*}
\frac{T}{4} \left( -\frac{1}{2} + q_0^2 + \frac{q_0}{1 - q_1} + (1 - q_1)^2 - q_0 (1 - q_1) \right) + O \bigl( \sqrt{T} \bigr) .
\end{align*}
A suitable choice of $q_0$ is $q_0 = (1 - q_1)^2$, which yields
\begin{align*}
&\frac{T}{4} \left( -\frac{1}{2} + (1 - q_1)^4 + (1 - q_1) + (1 - q_1)^2 - (1 - q_1)^3 \right) + O \bigl( \sqrt{T} \bigr) \\
&\leq \frac{T}{4} \left( -\frac{1}{2} + (1 - q_1) + (1 - q_1)^2 \right) + O \bigl( \sqrt{T} \bigr) .
\end{align*}
We can take the moderate choice $q_1 = \frac{3}{4}$ for example (and hence $q_0 = \frac{1}{16}$), yielding the upper bound 
$-\frac{3 T}{64} + O \bigl( \sqrt{T} \bigr)$.
\end{proof}
	
	\subsection{Online mirror descent}
	
	For OLO over the simplex, the online mirror descent method with entropic regularization and {\em constant learning rate}, attains best-case lower bounds $-O(\sqrt{T\log d})$  \citep{blum2018preserving}. This algorithm is known to attain upper bounds on the regret of the same order. The extension of this result for {\em non-increasing learning rates},\footnote{In fact, we have found a simple argument showing the regret is non-negative, strengthening the result in \cite{blum2018preserving}.} albeit possible, seems of limited interest, as this algorithm is known to achieve linear regret in the worst case \citep{orabona2018scale}, which is attributed to the unboundedness of the Bregman divergence.

	\section{Best-case loss sequence for binary DTOL with two experts}
	\label{sec:best-case-seq}
	
In this section, we characterize the binary loss sequence for DTOL with two experts in which the regret for Decreasing Hedge is minimized. Recall that Decreasing Hedge chooses expert $i$ at round $t$ 
with probability $p_{i,t}= \frac{\exp(-\eta_t L_{i,t-1})}{\sum_{j \in \{1,2\}} \exp(-\eta_t L_{j,t-1})}$, where $\eta_t = \sqrt{1/t}$. Note that $\eta_{t+1} < \eta_{t}$ for all $t$. We denote the loss vector for round $t$ as $\lossv_t = \binom{\loss_{1,t}}{\loss_{2,t}}$, where $\loss_{1,t} , \loss_{2,t} \in \{0,1\}$.  We denote $\binom{L_{t,1}}{L_{t,2}}$ by  $\bm{L}_{t}$.

Our approach is to introduce a set of operations on a sequence of binary loss vectors $\lossv_1, \lossv_2, \ldots, \lossv_T$ that do not increase the regret of Decreasing Hedge. Using this set of operations, we will show that any loss sequence with $T=2K$ rounds $\lossv_1, \lossv_2, \ldots, \lossv_{2K}$, can be converted to a specific loss sequence which we call the canonical best-case sequence. 
\begin{definition}
  For any length   $T=2K\geq 16$, 
  the canonical best-case loss sequence is defined as:
  \begin{align} \textstyle
    \mathrm{Canonical}(2K) = {\binom{0}{1}}^{K-1} {\binom{0}{0}}^2 {\binom{1}{0}}^{K-1} . 
  \end{align}
\end{definition}
\vspace{-0.2cm}

\subsection{Converting all (1,1) to (0,0)}
\label{sc:converting_1_to_0_section}
	
	In this part, we will show that adding a constant term to a loss sequence at round $t$ would not have any effect on the regret. This is because in Decreasing Hedge, the weights in each round $t$ depend only on the difference between the cumulative loss of two experts up until the end of round $t-1$. Adding a constant loss to both experts at any round would not change the difference between cumulative losses at any round; therefore, the weight vector remains the same for all rounds. Moreover, in round $t$, the best expert as well as Decreasing Hedge incur the same additional loss; therefore, the regret would be the same in that round.
	\begin{proposition}
		Adding constant value $c$ to the loss of all experts at round $t$ would not have any effect on the regret. 
		\label{prp:adding_constant}
	\end{proposition}

	Using Proposition \ref{prp:adding_constant}, we can substitute all (1,1) loss vectors with (0,0) loss vectors. This means that there is no need for (1,1) loss vectors in the best-case sequence. 
	
	\subsection{Dealing with Leader Change} 
	\label{sc:leader_change_section}
	
		In this subsection, we introduce the notion of a leader change and show that there is no leader change in the best-case sequence. Therefore, we only need to consider loss sequences without leader changes.
		
		 Formally, we say expert $j$ is a leader at round $t$ if 
	$L_{j,t} = \min_{i \in \{1,2\}} L_{i,t}$. Moreover, if $j$ is the \emph{only} leader at round $t$, then we define the \emph{strict leader} at round $t$ to be $i^*_t=j$.\footnote{If in round $t$, $L_{1,t}=L_{2,t}$, then the strict leader is not defined in that round.}

	We say that a sequence has a leader change if there exists times $t_1$ and $t_2$ such that in both times the strict leader is defined and $i^*_{t_1}\neq i^*_{t_2}$. Defining $\Delta_t := L_{2,t} - L_{1,t}$, observe that if a sequence has a leader change, then the sign of $\Delta$ should change at least once. In the following, we will show that any loss sequence $\lossv_1, \lossv_2, \ldots, \lossv_T$ can be converted, without increasing the regret, to a sequence $\lossv'_1, \lossv'_2, \ldots, \lossv'_T$ where  $\Delta$ is always non-negative (hence, the resulting sequence has no leader change). Removing all leader changes facilitates the next operation to successfully convert loss sequences.

	We first need to define an operation called switching the loss at round $t$. We define $\bar{i} :=  \begin{cases} 
		2 & i = 1 \\
		1 & i = 2 \\ 
	\end{cases}$.
	\begin{definition}
		The operation of \emph{switching loss sequence $\lossv_1, \lossv_2, \ldots, \lossv_T$
			at round $t$} yields a new loss sequence $\lossv'_1, \lossv'_2, \ldots, \lossv'_T$ where for all $s < t$, we have $\loss'_{i,s}=\loss_{i,s}$ and for all $s \geq t$, we have $\loss'_{i,s}=\loss_{\bar{i},s}$. 
		\label{df:switching_operation}
	\end{definition}

	Obviously switching loss sequence  $\lossv_1, \lossv_2, \ldots, \lossv_T$ at round $1$ only swaps the indices of the experts; therefore, the regret remains the same for the switched loss sequence. Similarly, we will show that for any $t$, where $\Delta_{t-1}=0$, switching loss sequence  $\lossv_1, \lossv_2, \ldots, \lossv_T$ at round $t$ does not change the regret.

	\begin{lemma}
		For a loss sequence $\lossv_1, \lossv_2, \ldots, \lossv_T$, if $\Delta_{t-1}=0$ then switching loss sequence at round $t$ does not change the regret. 
		\label{op:switching_operation_rule}	
	\end{lemma}
	
	The proof is in Appendix~\ref{prf:switching_operation_rule}.

	We also need to formally define the notion of a leader change at round $t$. 
	\begin{definition}
		We say a loss sequence has a leader change at round $t$ if there exists $k<t$ such that for all $s$ satisfying $t-k<s<t$, it holds that $\Delta_{s}=0 $ and $\Delta_{t}\cdot \Delta_{t-k} < 0 $.
	\end{definition}
	\begin{definition}
		Consider a sequence $\lossv_1, \lossv_2, \ldots, \lossv_T$ where $t$ is the smallest value such that $\Delta_t \neq 0$. We define the first strict leader in this sequence to be expert $1$ if $\Delta_t > 0$ and expert $2$ if $\Delta_t < 0$. 
	\end{definition}
	We are now ready to show that any sequence with leader changes can be converted, without increasing the regret, to a sequence with no leader change. We only consider loss sequences in which the first strict leader is expert $1$ because if the first strict leader for a sequence $\lossv_1, \lossv_2, \ldots, \lossv_T$ is expert $2$, then we can simply swap the experts' indices.
	
	Now, if the first strict leader is expert $1$, then $\Delta$ first becomes positive before it gets any chance to become negative. Our goal now is to modify the loss sequence without increasing the regret so that $\Delta$ stays non-negative for all rounds. The number of times the sign of $\Delta$ changes corresponds to the number of leader changes.
	\begin{observation}
		The number of leader changes in a sequence is the number of times the sign of $\Delta_t$ changes from negative to positive or vice versa.  
	\end{observation}

	\begin{lemma}
		If a loss sequence $\lossv_1, \ldots, \lossv_T$ has at least one leader change, and the first leader change happens at round $t$, then switching the loss sequence at round $t$ would remove the first leader change without changing the regret. Moreover, this switch operation ensures that the first leader change (if any) in resulting loss sequence $\lossv'_1, \ldots, \lossv'_T$ happens at time $s>t$.  
		\label{eq:decrease_leader_change_by_1}
	\end{lemma}
	Figure \ref{fig:decreasing_leader_change} illustrates how applying Lemma \ref{eq:decrease_leader_change_by_1} removes the first leader change. In the following, we will give a formal proof.
	\begin{proof}
		We assume that the first strict leader is expert $1$, therefore $\Delta_s \geq 0$ for $s < t$. Now, in round $t$ where the first leader change happens, we know that $\Delta_t< 0$. Moreover, $\Delta_{t-1} = 0$. Now by Lemma \ref{op:switching_operation_rule} we can switch the loss sequence at round $t$ without increasing the regret. In the resulting loss sequence, for all $s<t$, $\Delta'_s = \Delta_s$ and for all $s \geq t$, $\Delta'_s = -\Delta_s$. This ensures that for all $s \leq t$, $\Delta'_t \geq 0$. Therefore, in the resulting loss sequence, no leader change could happen at any time $s\leq t$. Therefore, the first leader change (if any) in the resulting loss sequence happens at some time $s>t$. 
	\end{proof}

	\begin{figure}[h!]
		\centering
		\begin{subfigure}[b]{0.45\linewidth}
			\includegraphics[width=\linewidth]{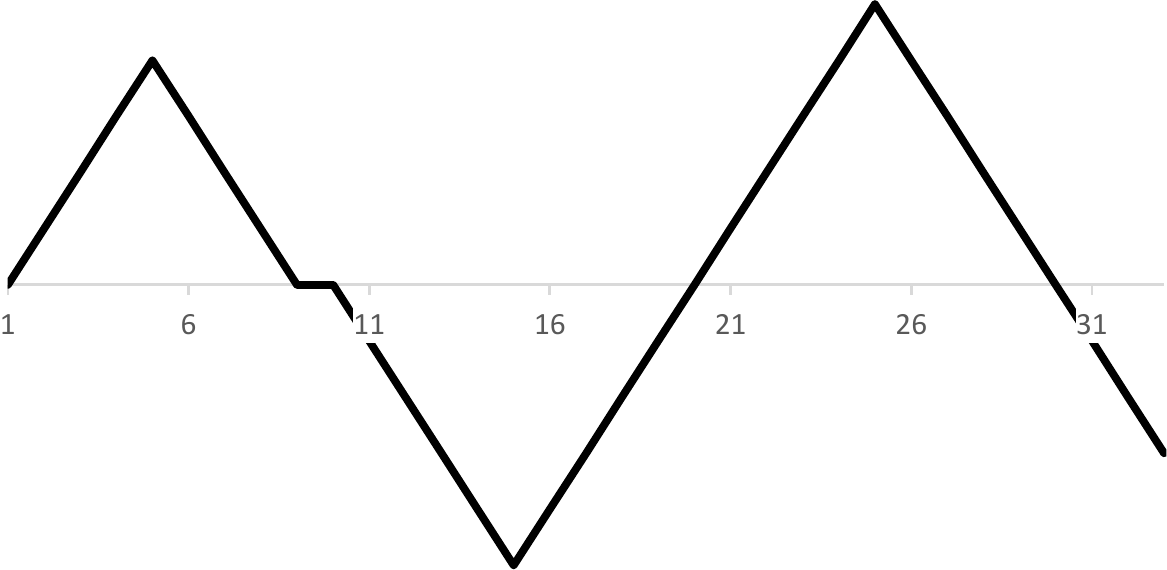}
			\caption{$\Delta_t$.}
		\end{subfigure}
		\hspace{0.05\linewidth}
		\begin{subfigure}[b]{0.45\linewidth}
			\includegraphics[width=\linewidth]{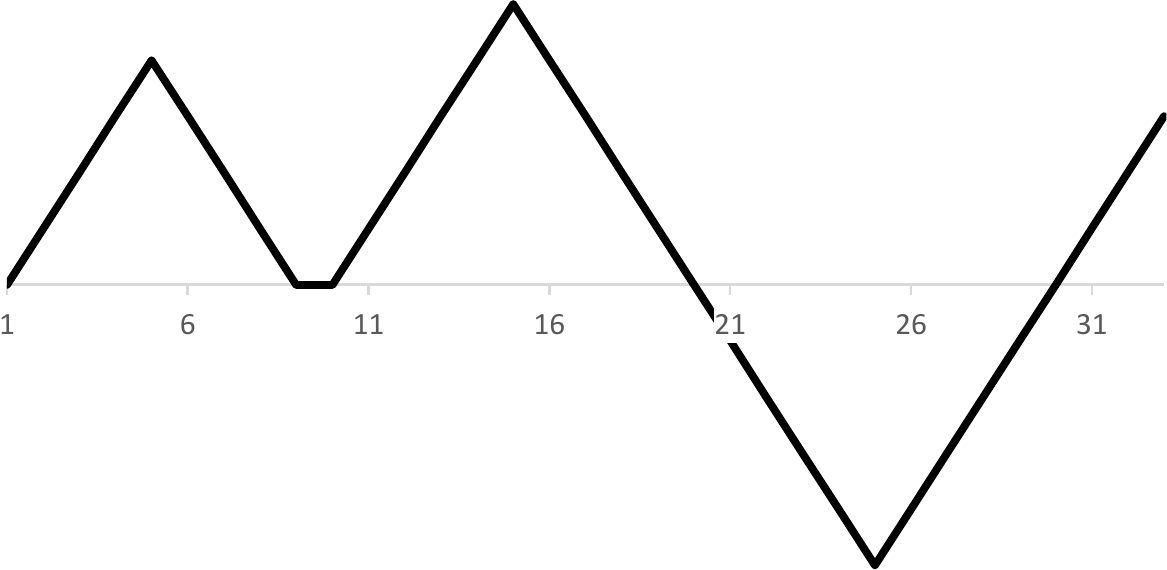}
			\caption{$\Delta'_t$  after switching at round $11$}
		\end{subfigure}
		\caption{The plots show $\Delta_t$ and $\Delta'_t$ for a loss sequence $\lossv_1, \ldots, \lossv_T$ and $\lossv'_1, \ldots, \lossv'_T$. The first leader change in $\lossv_1, \ldots, \lossv_T$ happens at round 11. Switching at round 11 decreases the number of leader changes. The first leader change in the resutling sequence $\lossv'_1, \ldots, \lossv'_T$ happens at round $21>11$.}
		\label{fig:decreasing_leader_change}
	\end{figure}
	
	\begin{lemma}
		Any loss sequence with a leader change can be converted to a loss sequence with no leader change without changing the regret.
		\label{lm:no_leader_change_sequence}
	\end{lemma}
	\begin{proof}
		If a sequence has at least one leader change, we can remove the first leader change at round $t$ using Lemma \ref{eq:decrease_leader_change_by_1}. By applying this lemma, the first leader change (if any) in the resulting loss sequence happens at round $s>t$. Because the number of rounds $T$ is bounded, we can use this operation finitely many times to reach a sequence without any leader change.
	\end{proof}

	\subsection{Swapping two consecutive losses}
	
	In this section, we will consider the operation of swapping two consecutive loss vectors. When we swap a loss sequence $\lossv_1, \ldots,\lossv_{t-1}, \lossv_t, \lossv_{t+1}, \lossv_{t+2}, \ldots, \lossv_T$ at rounds $(t, t+1)$, the resulting sequence becomes $\lossv_1, \ldots,\lossv_{t-1}, \lossv_{t+1}, \lossv_{t}, \lossv_{t+2} \ldots, \lossv_T$. Because the behavior of Decreasing Hedge at round $s$ only depends on $\bm{L}_{s-1}$, swapping two consecutive loss vectors at rounds $(t,t+1)$ does not change $\bm{L}_{s-1}$ for $s \in [T] \setminus \{t, t+1\}$ and therefore does not affect the behavior of the algorithm at any rounds other than rounds $t$ and $t+1$. Therefore, in order to see the effect of swapping two consecutive loss vectors at rounds $(t,t+1)$, it is enough to compare $\hedgeloss_{t} + \hedgeloss_{t+1}$ with $\hedgeloss'_{t} + \hat{\loss}'_{t+1}$ where $\hedgeloss'_{t}$ is defined to be the loss incurred by Decreasing Hedge at round $t$ in the modified sequence.

	In Section \ref{sc:converting_1_to_0_section}, we had shown that we can replace all $\binom{1}{1}$ loss vectors with $\binom{0}{0}$ without changing the regret. Thus, we only consider all possible sequences with loss vectors $\binom{1}{0}$, $\binom{0}{1}$, $\binom{0}{0}$ and characterize when swapping two consecutive loss vectors of the above form does not increase the regret. 
	
	Moreover, by Section \ref{sc:leader_change_section}, we only need to consider loss sequences $\lossv_1, \ldots, \lossv_T$ that do not have a leader change prior to the swap. Moreover, without loss of generality, we can assume the expert 1 is always the leader in the loss sequence.  
	
	We will now show that in a sequence where expert 1 is always the leader, any swap that pushes the $\binom{1}{0}$ vector to the later rounds or any swap that pushes the $\binom{0}{1}$ to the earlier rounds does not increase the regret.

	\begin{lemma}[Swapping rules]
		In a loss sequence $\lossv_1, \ldots, \lossv_T$, where $\Delta_{t-1} \geq 0$, swapping at round $(t,t+1)$ does not increase the regret in any of the following cases: 
		\begin{enumerate}[label=(\alph*)]
			\item $\loss_{t}, \loss_{t+1} = \binom{1}{0} \binom{0}{0};$ 
			\item $\loss_{t}, \loss_{t+1} = \binom{0}{0} \binom{0}{1};$ 
			\item $\loss_{t}, \loss_{t+1} = \binom{1}{0} \binom{0}{1}.$
		\end{enumerate}
		\label{lm:swapping_rules_lemma_from_Nishant}
	\end{lemma}
	
	The proof is in Appendix~\ref{prf:swapping_rules_lemma_from_Nishant}.

	Observe that by Lemma \ref{lm:swapping_rules_lemma_from_Nishant} we can always move $\binom{0}{1}$ to the earlier rounds and $\binom{1}{0}$ to the later rounds. By repeatedly applying this swapping rule, the resulting loss sequence will be of the form $\lossv_1, \ldots, \lossv_T = {\binom{0}{1}}^a {\binom{0}{0}}^c {\binom{1}{0}}^b$. 
	Note that as we have shown that $\Delta$ is always non-negative, $a\geq b$. 

	\subsection{Regret is minimized when both experts have the same cumulative loss}
	So far we have shown that any loss sequence can be converted to a loss sequence with form $\left({\binom{0}{1}}^{a} {\binom{0}{0}}^c {\binom{1}{0}}^{b}\right)$ where $a\geq b$ while having the same or even less regret for Decreasing Hedge. First we show that when $T$ is even, the number of zeros in the middle should be even.

	\begin{lemma}
		If a sequence has the form $\lossv_1, \ldots, \lossv_T = \left({\binom{0}{1}}^{a-1} \binom{0}{1} {\binom{0}{0}}^{2k-1} {\binom{1}{0}}^{b}\right)$ where $a\geq b$ and $T$ is even, then converting it to a loss sequence  $\lossv'_1, \ldots, \lossv'_T=\left({\binom{0}{1}}^{a-1} \binom{0}{0} {\binom{0}{0}}^{2k-1} {\binom{1}{0}}^{b}\right)$ decreases the regret.
		\label{lm:c_even}
	\end{lemma}
	
	The proof is in Appendix~\ref{prf:c_even}.

	Now, we only consider loss sequences of form  $\left({\binom{0}{1}}^{a} {\binom{0}{0}}^c {\binom{1}{0}}^{b}\right)$ where $c$ is even. As $a+b+c=T$ and $T$ is even, $a+b=2K$ is even. 
	
	Next, we will show in the following lemma that any loss sequence $\left({\binom{0}{1}}^{a} {\binom{0}{0}}^c {\binom{1}{0}}^{b}\right)$ where $a+b=2K$ and $a\geq b$ can be converted to a loss sequence $\left({\binom{0}{1}}^{K} {\binom{0}{0}}^c {\binom{1}{0}}^{K}\right)$ while having the same or even less regret. 
	\begin{lemma}
		Among all loss sequences of form $\left({\binom{0}{1}}^{a} {\binom{0}{0}}^c {\binom{1}{0}}^{b}\right)$ where $a+b=2K$ and $a\geq b$, loss sequence $\left({\binom{0}{1}}^{K} {\binom{0}{0}}^c {\binom{1}{0}}^{K}\right)$ has the least regret for Decreasing Hedge.
		\label{lm:equalize_both_experts_cumulative_loss}
	\end{lemma}
The proof is in Appendix~\ref{prf:equalize_both_experts_cumulative_loss}.

	\subsection{The optimum number of $(0,0)$ loss vectors in the middle}

	So far, we know that the best-case sequence has the form $\left({\binom{0}{1}}^{K} {\binom{0}{0}}^c {\binom{1}{0}}^{K}\right)$ where $c$ is even. Is the existence of $\binom{0}{0}^c$ in the middle necessary to have the loss sequence with the least possible regret?

	In order to examine this, we consider a simple modification to the loss sequence. Consider a loss sequence in which one $(0,0)$ is replaced with $(0,1)$ and another $(0,0)$ is replaced with $(1,0)$. Then by the swapping lemma (Lemma \ref{lm:swapping_rules_lemma_from_Nishant}), we can move the $(1,0)$ to the later rounds and the $(0,1)$ to the earlier rounds without increasing the regret. If we denote the original sequence as $\lossv_1,\ldots,\lossv_T = {\binom{0}{1}}^a {\binom{0}{0}}^c {\binom{1}{0}}^b$, the modified loss sequence becomes  $\lossv'_1,\ldots,\lossv'_T={\binom{0}{1}}^{a+1} {\binom{0}{0}}^{c-2} {\binom{1}{0}}^{b+1}$.  
	\begin{lemma}
		Consider a loss sequence $\lossv_1,\ldots,\lossv_T = {\binom{0}{1}}^a {\binom{0}{0}}^c {\binom{1}{0}}^b$, where $c$ is even and $a \geq 4$. Then modifying loss sequence to $\lossv'_1,\ldots,\lossv'_T={\binom{0}{1}}^{a+1} {\binom{0}{0}}^{c-2} {\binom{1}{0}}^{b+1}$ does not increase the regret if $c \geq 4$ and does increase the regret if $c \leq 3$.
		\label{lm:remove_zeros}
	\end{lemma}
The proof is in Appendix~\ref{prf:remove_zeros}.

	Now, let us consider a loss sequence $\lossv_1, \ldots, \lossv_T$ of length $T = 2K \geq 16$ of the form ${\binom{0}{1}}^a {\binom{0}{0}}^c {\binom{1}{0}}^b$. If $a > b$, then we can and will use Lemma~\ref{lm:equalize_both_experts_cumulative_loss} to convert this sequence into a sequence for which $a = b$. The resulting sequence has the form ${\binom{0}{1}}^a {\binom{0}{0}}^c {\binom{1}{0}}^a$. We will now show how this sequence can be modified, without increasing the regret, to a sequence for which $c = 2$ (while maintaining the invariant $a = b$). There are two cases.
	
	In the first case, the sequence is of the form ${\binom{0}{1}}^a {\binom{0}{0}}^c {\binom{1}{0}}^a$ where $ a \geq 4$ and $c\geq 4$ (recall that $c$ must be even, and note that if $c =2$, we are done). We may then apply Lemma~\ref{lm:remove_zeros} to decrease the number of $(0,0)$ vectors from the middle by two while maintaining the invariant $a = b$. As we have shown in Lemma \ref{lm:c_even} that $c$ is even, the last time where we can apply this lemma is when $c=4$ and the resulting sequence is of the form ${\binom{0}{1}}^{a + \frac{c-2}{2}} {\binom{0}{0}}^2 {\binom{1}{0}}^{\frac{c-2}{2} + a}$
	
	In the second case, we have $a \in \{1,2,3\}$. Since $a+b+c=T \geq 16$ and as we may assume that $a = b$, it follows that $c \geq 10$. In this case, inequality \eqref{eq:inequality_for_removing_two_zeros} in the proof of Lemma \ref{lm:remove_zeros} holds, which implies that modifying the loss sequence to  $\lossv'_1,\ldots,\lossv'_T={\binom{0}{1}}^{a+1} {\binom{0}{0}}^{c-2} {\binom{1}{0}}^{a+1}$ does not increase the regret. After applying this modification at most three times, the resulting loss sequence will be of the form $\lossv''_1, \ldots, \lossv''_T= {\binom{0}{1}}^a {\binom{0}{0}}^c {\binom{1}{0}}^a$ where $ a \geq 4$. We are now in the first case and so again can arrive at a sequence with two $(0, 0)$ vectors in the middle.
	
	Observe that  ${\binom{0}{1}}^{a + \frac{c-2}{2}} {\binom{0}{0}}^2 {\binom{1}{0}}^{\frac{c-2}{2} + a}$ is of the form ${\binom{0}{1}}^{K} {\binom{0}{0}}^2 {\binom{1}{0}}^{K}$. This form coincides with the canonical best-case sequence we mentioned earlier.

\subsection{Bounding the regret}
As shown in Appendix~\ref{sc:calculating_the_regret}, the regret on the canonical best-case loss sequence can be lower and upper bounded as follows:

\begin{align*}
-\frac{e^2}{ (1-\frac{1}{e})} \sqrt{T}  -  \frac{1}{2} 
\leq \regret(T) 
\leq -\frac{1}{1+e^{\sqrt{2}}}\sqrt{T}  +  \frac{12}{\sqrt{e}} +\frac{1}{2} .
\end{align*}
As both the lower and upper bounds are $-\Theta(\sqrt{T})$, the analysis in Section~\ref{sec:general} is tight in this case.

\section{Discussion}
\label{sec:discussion}
	
In this work, we provided a systematic treatment of best-case lower bounds in online learning for adaptive FTRL algorithms, discussed the impossibility of certain natural extensions, also provided a tighter analysis of such lower bounds in the binary prediction experts setting. As one application, we have shown that our results for adaptive FTRL enable the use a broad class of adaptive online learning algorithms that satisfy a balanced mistakes notion of group fairness. 
Naturally, many questions still remain open, and we hope this work motivates further research in this intriguing topic. 
A first question relates to algorithms that can achieve negative regret: can we characterize the conditions under which this happens? The goal would be to reveal, beyond the particular example we gave in Section~\ref{sec:linearized-ftrl}, structural properties of an instance that lead to substantial outperformance of the best fixed action. Returning to that example, another question arises. We have shown an example of OCO with a strongly convex loss (the squared loss) for which linearized FTRL obtains negative linear regret, whereas our best-case lower bounds for (non-linearized) FTRL with constant learning rate imply $-O(\sqrt{T})$ regret and known upper bounds imply  $O(\log T)$ regret. Thus, while in this situation, linearized FTRL pays a price with respect to regret upper bounds, it can exhibit a switching behavior that might allow it to compete with a shifting sequence of comparators. Investigating this ``blessing'' of linearization would be a fascinating investigation for future work. Finally, in the setting of DTOL with two experts, we showed that our best-case lower bounds are tight via explicit construction of the best-case sequence. It would be interesting to see if a proof of tightness could be obtained without resorting to an explicit construction, as this could allow us to say that the best-case lower bounds are achievable more generally. For instance, whereas our swapping-based technique is most related to that of \cite{vanerven2014follow} for worst-case sequences and 
\cite{lewi2020thompson} (who give both worst and best-case sequences), it could also be worthwhile to draw from previous works \citep{takimoto2000minimax,koolen2014efficient,bartlett2015minimax,koolen2015minimax} which explicitly work out the minimax regret and minimax strategies. 
	
\subsubsection*{Acknowledgements}

CG’s research is partially supported by INRIA through the INRIA Associate Teams project and FONDECYT 1210362 project. NM and AM were supported by the NSERC Discovery Grant RGPIN-2018-03942. We also thank Sajjad Azami for his involvement in the earlier stages of this work.

\bibliography{best_case_lower_bounds}

\begin{thebibliography}{20}
\providecommand{\natexlab}[1]{#1}
\providecommand{\url}[1]{\texttt{#1}}
\expandafter\ifx\csname urlstyle\endcsname\relax
  \providecommand{\doi}[1]{doi: #1}\else
  \providecommand{\doi}{doi: \begingroup \urlstyle{rm}\Url}\fi

\bibitem[Abernethy et~al.(2008)Abernethy, Hazan, and
  Rakhlin]{abernethy2008competing}
Jacob Abernethy, Elad~E Hazan, and Alexander Rakhlin.
\newblock Competing in the dark: An efficient algorithm for bandit linear
  optimization.
\newblock In \emph{21st Annual Conference on Learning Theory, COLT 2008}, pages
  263--273, 2008.

\bibitem[Auer et~al.(2002)Auer, Cesa-Bianchi, and Gentile]{auer2002adaptive}
Peter Auer, Nicolo Cesa-Bianchi, and Claudio Gentile.
\newblock Adaptive and self-confident on-line learning algorithms.
\newblock \emph{Journal of Computer and System Sciences}, 64\penalty0
  (1):\penalty0 48--75, 2002.

\bibitem[Bartlett et~al.(2015)Bartlett, Koolen, Malek, Takimoto, and
  Warmuth]{bartlett2015minimax}
Peter~L Bartlett, Wouter~M Koolen, Alan Malek, Eiji Takimoto, and Manfred~K
  Warmuth.
\newblock Minimax fixed-design linear regression.
\newblock In \emph{Conference on Learning Theory}, pages 226--239. PMLR, 2015.

\bibitem[Blum et~al.(2018)Blum, Gunasekar, Lykouris, and
  Srebro]{blum2018preserving}
Avrim Blum, Suriya Gunasekar, Thodoris Lykouris, and Nati Srebro.
\newblock On preserving non-discrimination when combining expert advice.
\newblock In S.~Bengio, H.~Wallach, H.~Larochelle, K.~Grauman, N.~Cesa-Bianchi,
  and R.~Garnett, editors, \emph{Advances in Neural Information Processing
  Systems}, volume~31. Curran Associates, Inc., 2018.
\newblock URL
  \url{https://proceedings.neurips.cc/paper/2018/file/2e855f9489df0712b4bd8ea9e2848c5a-Paper.pdf}.

\bibitem[Cesa-Bianchi and Lugosi(2006)]{cesabianchi2006prediction}
Nicolo Cesa-Bianchi and Gabor Lugosi.
\newblock \emph{Prediction, learning, and games}.
\newblock Cambridge university press, 2006.

\bibitem[Chernov and Zhdanov(2010)]{chernov2010prediction}
Alexey Chernov and Fedor Zhdanov.
\newblock Prediction with expert advice under discounted loss.
\newblock In \emph{International Conference on Algorithmic Learning Theory},
  pages 255--269. Springer, 2010.

\bibitem[De~Rooij et~al.(2014)De~Rooij, Van~Erven, Gr{\"u}nwald, and
  Koolen]{derooij2014follow}
Steven De~Rooij, Tim Van~Erven, Peter~D Gr{\"u}nwald, and Wouter~M Koolen.
\newblock Follow the leader if you can, hedge if you must.
\newblock \emph{The Journal of Machine Learning Research}, 15\penalty0
  (1):\penalty0 1281--1316, 2014.

\bibitem[Duchi et~al.(2011)Duchi, Hazan, and Singer]{duchi2011adaptive}
John Duchi, Elad Hazan, and Yoram Singer.
\newblock Adaptive subgradient methods for online learning and stochastic
  optimization.
\newblock \emph{Journal of Machine Learning Research}, 12:\penalty0 2121--2159,
  2011.

\bibitem[Freund and Schapire(1997)]{freund1997decision}
Yoav Freund and Robert~E Schapire.
\newblock A decision-theoretic generalization of on-line learning and an
  application to boosting.
\newblock \emph{Journal of computer and system sciences}, 55\penalty0
  (1):\penalty0 119--139, 1997.

\bibitem[Gofer and Mansour(2016)]{gofer2016lower}
Eyal Gofer and Yishay Mansour.
\newblock Lower bounds on individual sequence regret.
\newblock \emph{Machine Learning}, 103\penalty0 (1):\penalty0 1--26, 2016.

\bibitem[Herbster and Warmuth(1998)]{herbster1998tracking}
Mark Herbster and Manfred~K Warmuth.
\newblock Tracking the best expert.
\newblock \emph{Machine learning}, 32\penalty0 (2):\penalty0 151--178, 1998.

\bibitem[Kivinen and Warmuth(1997)]{kivinen1997exponentiated}
Jyrki Kivinen and Manfred~K Warmuth.
\newblock Exponentiated gradient versus gradient descent for linear predictors.
\newblock \emph{information and computation}, 132\penalty0 (1):\penalty0 1--63,
  1997.

\bibitem[Koolen et~al.(2014)Koolen, Malek, and Bartlett]{koolen2014efficient}
Wouter~M Koolen, Alan Malek, and Peter~L Bartlett.
\newblock Efficient minimax strategies for square loss games.
\newblock In Z.~Ghahramani, M.~Welling, C.~Cortes, N.~Lawrence, and K.~Q.
  Weinberger, editors, \emph{Advances in Neural Information Processing
  Systems}, volume~27. Curran Associates, Inc., 2014.
\newblock URL
  \url{https://proceedings.neurips.cc/paper/2014/file/8d420fa35754d1f1c19969c88780314d-Paper.pdf}.

\bibitem[Koolen et~al.(2015)Koolen, Malek, Bartlett, and
  Abbasi~Yadkori]{koolen2015minimax}
Wouter~M Koolen, Alan Malek, Peter~L Bartlett, and Yasin Abbasi~Yadkori.
\newblock Minimax time series prediction.
\newblock In C.~Cortes, N.~Lawrence, D.~Lee, M.~Sugiyama, and R.~Garnett,
  editors, \emph{Advances in Neural Information Processing Systems}, volume~28.
  Curran Associates, Inc., 2015.
\newblock URL
  \url{https://proceedings.neurips.cc/paper/2015/file/4dcf435435894a4d0972046fc566af76-Paper.pdf}.

\bibitem[Lewi et~al.(2020)Lewi, Kaplan, and Mansour]{lewi2020thompson}
Yuval Lewi, Haim Kaplan, and Yishay Mansour.
\newblock Thompson sampling for adversarial bit prediction.
\newblock In \emph{Algorithmic Learning Theory}, pages 518--553. PMLR, 2020.

\bibitem[McMahan(2017)]{mcmahan2017survey}
H.~Brendan McMahan.
\newblock A survey of algorithms and analysis for adaptive online learning.
\newblock \emph{Journal of Machine Learning Research}, 18\penalty0
  (90):\penalty0 1--50, 2017.
\newblock URL \url{http://jmlr.org/papers/v18/14-428.html}.

\bibitem[Mourtada and Ga{\"\i}ffas(2019)]{mourtada2019optimality}
Jaouad Mourtada and St{\'e}phane Ga{\"\i}ffas.
\newblock On the optimality of the hedge algorithm in the stochastic regime.
\newblock \emph{Journal of Machine Learning Research}, 20:\penalty0 1--28,
  2019.

\bibitem[Orabona and P{\'a}l(2018)]{orabona2018scale}
Francesco Orabona and D{\'a}vid P{\'a}l.
\newblock Scale-free online learning.
\newblock \emph{Theoretical Computer Science}, 716:\penalty0 50--69, 2018.

\bibitem[Takimoto and Warmuth(2000)]{takimoto2000minimax}
Eiji Takimoto and Manfred~K Warmuth.
\newblock The minimax strategy for gaussian density estimation. pp.
\newblock In \emph{Proceedings of the Thirteenth Annual Conference on
  Computational Learning Theory}, pages 100--106, 2000.

\bibitem[Van~Erven et~al.(2014)Van~Erven, Kot{\l}owski, and
  Warmuth]{vanerven2014follow}
Tim Van~Erven, Wojciech Kot{\l}owski, and Manfred~K Warmuth.
\newblock Follow the leader with dropout perturbations.
\newblock In \emph{Conference on Learning Theory}, pages 949--974. PMLR, 2014.

\end{thebibliography}

\appendix

	\section{Derivations for \eqref{eqn:timeless-dtol-bclb}}
	\label{app:timeless}
	
	Recall that we set $\eta_t$ as
	\begin{align*}
		\eta_t = -\log \left( 1 - \min \left\{ \frac{1}{4}, \sqrt{\frac{2 \log d}{L^*_t}} \right\} \right) .
	\end{align*}
        Starting from \eqref{eqn:time-varying-eta} and using the fact that $(\eta_t)_{t \geq 0}$ is non-increasing and $\Phi$ (the negative Shannon entropy) is non-positive, we have
	\begin{align*}
          \regret_T 
          \geq \big( \frac{1}{\eta_T} - \frac{1}{\eta_0} \big) \inf_{w \in \actions} \Phi(w) .
	\end{align*}
        Next, since $-\log(1 - x) \geq x$ for $x < 1$, we have that $\frac{1}{\eta_T} \leq \max \left\{ 4, \sqrt{\frac{L^*_T}{2 \log d}} \right\}$. We also have $\eta_0 = 4$ since $L^*_0 = 0$. Using these two facts, together with $\inf_{w \in \Delta_d} \Phi(w) = -\log d$ yields
	\begin{align*}
		\regret_T \geq \left( 4 - \max \left\{ 4, \sqrt{\frac{L^*_T}{2 \log d}} \right\} \right) \log d .
	\end{align*}
	The result \eqref{eqn:timeless-dtol-bclb} follows from some basics manipulations.

\section{Negative result for Linearized FTRL}
\label{app:linearized-ftrl}

In this section, we prove the claims used in the proof of Theorem~\ref{thm:linearized-ftrl}.
	
\begin{proof}[Proof (of Claim 1)]
Consider an arbitrary round $t$ in the first half. Observe that $g_s = p_s \geq 0$ for all $s < t$; consequently, $G_{t-1} \geq 0$. The claim follows from the definition of $p_t$.
\end{proof}

\begin{proof}[Proof (of Claim 2)]
We will show that the number of rounds in the first segment is at most
\begin{align*}
t_1 := \left\lceil \left( \frac{1}{2 q_0} \left( \log \frac{1 - q_0}{q_0} + \sqrt{\log^2 \frac{1-q_0}{q_0} + 4 q_0^2} \right) \right)^2 \right\rceil .
\end{align*}
  
The idea behind the proof is that whenever $p_t \geq q_0$, the gradient $g_t$ will be sufficiently positive and, consequently, when $p_t$ is updated it will make sufficient progress towards dropping below $q_0$. We now formalize this intuition.
  
Let $\mathcal{T}_1\subseteq [T/2]$ be the set of round indices in the first segment. 
By definition of the first segment, we have $p_t \geq q_0$ for all $t \in \mathcal{T}_1$. Therefore, for $t$ such that $t-1 \in \mathcal{T}_1$,
\begin{align}
p_t 
&= \frac{1}{1 + \exp \left( \eta_t \sum_{s=1}^{t-1} g_t \right)} \nonumber \\
&\leq \frac{1}{1 + \exp \left( \eta_t \sum_{s=1}^{t-1} q_0 \right)} 
   = \frac{1}{1 + \exp \left( q_0 \frac{t-1}{\sqrt{t}} \right)} \label{eqn:claim-2-start}.
\end{align}
  
It suffices to find the smallest $t$ such that \eqref{eqn:claim-2-start} is at most $q_0$. 
This is equivalent to finding the smallest $t$ such that
\begin{align*}
q_0 \frac{t-1}{\sqrt{t}} \geq \log \frac{1 - q_0}{q_0} .
\end{align*}
Rearranging and making the replacement $u := \sqrt{t}$, it suffices to solve the quadratic equation
\begin{align*}
q_0 u^2 - \left( \log \frac{1 - q_0}{q_0} \right) u - q_0 = 0 .
\end{align*}
Taking the positive solution and rounding up yields that the smallest such $t$ is at most
\begin{align*}
\left\lceil \left( 
  \frac
  { \log \frac{1-q_0}{q_0} 
    + \sqrt{\log^2 \frac{1 - q_0}{q_0}  + 4 q_0^2}}
  { 2 q_0 } 
\right)^2 \right\rceil
= t_1 .
\end{align*}
\end{proof}

\begin{proof}[Proof (of Claim 3)]
  
Let $\mathcal{T}_3$ be the set of round indices in the third segment. 
Let $t$ be such that $t - 1 \in \mathcal{T}_3$ (note: this implies that $t - 1 \geq T/2$). Then Claims 1 and 2, together with the definition of segment 3, imply that
\begin{align}
p_t 
&= \frac{1}{1 + \exp \left( \eta_t \sum_{s=1}^{t-1} g_t \right)} \nonumber \\
&\geq \frac{1}{1 + \exp \left( \eta_t \left( t_1 \cdot \frac{1}{2} + \left( \frac{T}{2} - t_1 \right) \cdot q_0 - \left( t - 1 - \frac{T}{2} \right) \cdot (1 - q_1) \right) \right)} \label{eqn:claim-3-start}.
\end{align}
  
It suffices to find the smallest $t$ such that \eqref{eqn:claim-3-start} is at least $q_1$. 
Proceeding similarly to the proof of Claim 2 (and once again introducing $u := \sqrt{t}$ yields the quadratic equation
\begin{align*}
-(1 - q_1) u^2 
+ \left( \log \frac{q_1}{1 - q_1} \right) u 
+ \left( \frac{1 - q_1 + q_0}{2} \cdot T + t_1 \left( \frac{1}{2} - q_0 \right) + 1 - q_1 \right)
= 0
\end{align*}
Solving for $u$ yields
\begin{align*}
u = \frac{\log \frac{q_1}{1 - q_1} + \sqrt{\log^2 \frac{q_1}{1 - q_1} + 4 (1 - q_1) \left( \frac{1 - q_1 + q_0}{2} \cdot T + t_1 \left( \frac{1}{2} - q_0 \right) + 1 - q_1 \right)}}{2 (1 - q_1)} .
\end{align*}
  
From the above, we see that the smallest $t$ satisfying \eqref{eqn:claim-3-start} is equal to
\begin{align*}
\frac{1 - q_1 + q_0}{2 (1 - q_1)} \cdot T + O \left( \sqrt{T} \right) ,
\end{align*}
which proves the claim.
\end{proof}

	\section{Best-case loss sequence for binary DTOL with 2 experts}
        \label{app:best-case-seq}

	\subsection{Dealing with Leader Change}	
	\begin{proof}[Proof of Lemma \ref{op:switching_operation_rule}]
		We show that after switching at round $t$ where $\Delta_{t-1}=L_{2,t-1}-L_{1,t-1}=0$, the cumulative loss of Decreasing Hedge and the best expert remains the same. Therefore, the regret does not change.

		For any round $k\geq t$,
		\begin{align*}
			L'_{i,k} &:= \sum_{s=1}^{t-1}\loss'_{i,s} + \sum_{s=t}^{k}\loss'_{i,s}  =\sum_{s=1}^{t-1}\loss_{i,s} + \sum_{s=t}^{k}\loss_{\noti,s}& (\text{switched loss definition)}\\
			&= L_{i,t-1} + \sum_{s=t}^{k}\loss_{\bar{i},s} = L_{\noti,t-1} + \sum_{s=t}^{k}\loss_{\noti,s}&(\Delta_{t-1}=0)\\
			&= L_{\noti,k}.
		\end{align*} 
		Hence, for $k=T$ we have $L'_{1,T} = L_{2,T}$, and $L'_{2,T} = L_{1,T}$. Therefore, the cumulative loss of the best expert for both loss sequences is the same, but the index of the best expert is changed. Moreover, observe that for any $k \geq t$, 
		\begin{align*}
			\Delta'_k &:= L'_{2,k} - L'_{1,k} = L_{1,k} - L_{2,k} = -\Delta_k .
		\end{align*} 
		Also for round $t$ we know that
		\begin{align*}
			\Delta'_{t-1} &:= L'_{2,t-1} - L'_{1,t-1} = L_{2,t-1} - L_{1,t-1} = \Delta_{t-1} = 0 = -\Delta_{t-1} .
		\end{align*} Therefore, for $s$ such that $s \geq t$, we have $\Delta'_{s-1} = -\Delta_{s-1}$.

		Next, denote the loss of Decreasing Hedge for loss sequence $\lossv_1, \ldots, \lossv_T$ (and $\lossv'_1, \ldots, \lossv'_T$) at round $s$ by $\hat{\loss}_{s}$ (and $\hat{\loss}'_{s}$ respectively). Obviously, as there is no difference up until round $t$, for all rounds $s< t$, we have $\hedgeloss_{s} = \hedgeloss'_s$. Now, for any round $s \geq t$:
		\begin{align*}
			\hedgeloss'_{s} &= \sum_{i\in \{1,2\}} \frac{\exp(-\eta_s(L'_{i,s-1}))}{\exp(-\eta_s(L'_{i,s-1}))+\exp(-\eta_s(L'_{\noti,s-1}))} \loss'_{i,s} \\
			&= \sum_{i\in \{1,2\}} \frac{1}{1 + \exp\big(\eta_s(L'_{i,s-1} - L'_{\noti,s-1})\big)} \loss'_{i,s} \\
			&= \frac{1}{1 + \exp\big(\eta_s(L'_{1,s-1} - L'_{2,s-1})\big)} \loss'_{1,s} + 
			\frac{1}{1 + \exp\big(\eta_s(L'_{2,s-1} -L'_{1,s-1})\big)} \loss'_{2,s} \\
			&= \frac{1}{1 + \exp\big(\eta_s(-\Delta'_{s-1})\big)} \loss'_{1,s} + 
			\frac{1}{1 + \exp\big(\eta_s(\Delta'_{s-1})\big)} \loss'_{2,s} \\
			&=\frac{1}{1 + \exp\big(\eta_s(\Delta_{s-1})\big)} \loss_{2,s} + 
			\frac{1}{1 + \exp\big(\eta_s(-\Delta_{s-1})\big)} \loss_{1,s} \\
			&= \sum_{i\in \{1,2\}} \frac{1}{1 + \exp\big(\eta_s(L_{i,s-1} - L_{\noti,s-1})\big)} \loss_{i,s}  = \hedgeloss_{s}.
		\end{align*}
		Therefore, the cumulative loss of Decreasing Hedge for both loss sequences is the same. 
	\end{proof}
\label{prf:switching_operation_rule}

	\subsection{Swapping two consecutive losses}
	
	\begin{proof}[Proof of Lemma \ref{lm:swapping_rules_lemma_from_Nishant}]
		Let $p_t$ (and $p'_t$) be the probability of choosing arm $1$ in round $t$ in loss sequence $\lossv_1, \ldots, \lossv_T$ (and $\lossv'_1, \ldots, \lossv'_T$ respectively). Define an increasing function $f(x):= \frac{1}{1+\exp(-x)}$. Now, we can rewrite $p_t$ as follows:
		\begin{align*}
			p_t = \frac{1}{1+\exp(-\eta_t{\Delta_{t-1}})}=f(\eta_t \Delta_{t-1}).
		\end{align*}

		\begin{enumerate}[label=(\alph*)]
			\item In this case, the loss incurred by Decreasing Hedge for loss sequence $\lossv_1, \ldots, \lossv_T$ at rounds $t$ and $t+1$ can be expressed as 
			\begin{align*}
				\hedgeloss_{t} + \hedgeloss_{t+1} &= \underbrace{\binom{p_t}{1-p_{t}} \cdot \binom{1}{0} }_{\text{round } t} + \underbrace{\binom{p_{t+1}}{1-p_{t+1}} \cdot \binom{0}{0} }_{\text{round } t+1} = p_{t}.
			\end{align*}
			Similarly the loss incurred by Decreasing Hedge for loss sequence $\lossv'_1, \ldots, \lossv'_T$ at rounds $t$ and $t+1$ can be written as 
			\begin{align*}
				\hedgeloss'_{t} + \hedgeloss'_{t+1} &= \underbrace{\binom{p'_t}{1-p'_{t}} \cdot \binom{0}{0} }_{\text{round } t} + \underbrace{\binom{p'_{t+1}}{1-p'_{t+1}} \cdot \binom{1}{0} }_{\text{round } t+1} = p'_{t+1} .
			\end{align*}
			We know that Decreasing Hedge incurs the same loss at any round $s \in [T]\setminus\{t,t+1\}$. Therefore, 
			\begin{align*}
				\hat{L'}_T - \hat{L}_T &=\sum_{s=1}^T \hedgeloss'_s - \sum_{s=1}^T \hedgeloss_s \\
				&= \big( \hedgeloss'_{t} + \hedgeloss'_{t+1}\big) - \big(\hedgeloss_{t} + \hedgeloss_{t+1} \big)\\
				&= p'_{t+1}  - p_t \\
				&= f(\eta_{t+1}\Delta'_{t}) - f(\eta_{t}\Delta_{t-1})\\
				&= f\big(\eta_{t+1}(\Delta'_{t-1}+0)\big) - f(\eta_{t}\Delta_{t-1})\\
				&=f(\eta_{t+1}\Delta'_{t-1}) - f(\eta_{t}\Delta_{t-1})\\
				&=f(\eta_{t+1}\Delta_{t-1}) - f(\eta_{t}\Delta_{t-1})\\
				&\leq 0,
			\end{align*}
			where the last inequality comes from the fact in Decreasing Hedge, $\eta_{t+1} < \eta_{t}$ and $f$ is an increasing function.

			\item In this case, the loss incurred by Decreasing Hedge for loss sequence $\lossv_1, \ldots, \lossv_T$ at rounds $t$ and $t+1$ can be expressed as
			\begin{align*}
				\hedgeloss_{t} + \hedgeloss_{t+1} &= \underbrace{\binom{p_t}{1-p_{t}} \cdot \binom{0}{0} }_{\text{round } t} + \underbrace{\binom{p_{t+1}}{1-p_{t+1}} \cdot \binom{0}{1} }_{\text{round } t+1} = 1-p_{t+1}.
			\end{align*}
			Similarly the loss incurred by Decreasing Hedge for loss sequence $\lossv'_1, \ldots, \lossv'_T$ at rounds $t$ and $t+1$ can be written as
			\begin{align*}
				\hedgeloss'_{t} + \hedgeloss'_{t+1} &= \underbrace{\binom{p'_t}{1-p'_{t}} \cdot \binom{0}{1} }_{\text{round } t} + \underbrace{\binom{p'_{t+1}}{1-p'_{t+1}} \cdot \binom{0}{0} }_{\text{round } t+1} = 1-p'_{t} .
			\end{align*}
			We know that Decreasing Hedge incurs the same loss at any round $s \in [T]\setminus\{t,t+1\}$; therefore, 
			\begin{align*}
				\hat{L'}_T - \hat{L}_T &=\sum_{s=1}^T \hedgeloss'_s - \sum_{s=1}^T \hedgeloss_s \\
				&=\big( \hedgeloss'_{t} + \hedgeloss'_{t+1}\big) - \big(\hedgeloss_{t} + \hedgeloss_{t+1} \big) \\
				&= (1 -  p'_t) - (1-p_{t+1}) \\
				&= p_{t+1}  - p'_t \\
				&= f(\eta_{t+1}\Delta_{t}) - f(\eta_{t}\Delta'_{t-1})\\
				&= 
				f\big(\eta_{t+1}(\Delta_{t-1}+0)\big) - f(\eta_{t}\Delta'_{t-1}) \\
				&=f(\eta_{t+1}\Delta_{t-1}) - f(\eta_{t}\Delta_{t-1})\\
				&\leq 0,
			\end{align*}
			where the last inequality comes from the fact in Decreasing Hedge, $\eta_{t+1} < \eta_{t}$ and $f$ is an increasing function.

			\item In this case, as we assume that expert $1$ is always the leader, $\Delta_s \geq 0$ for all $s$. Therefore, $\Delta_t = \Delta_{t-1} - 1 \geq 0$. The loss incurred by Decreasing Hedge for loss sequence $\lossv_1, \ldots, \lossv_T$ at rounds $t$ and $t+1$ can be expressed as
			\begin{align*}
				\hedgeloss_{t} + \hedgeloss_{t+1} &= \underbrace{\binom{p_t}{1-p_{t}} \cdot \binom{1}{0} }_{\text{round } t} + \underbrace{\binom{p_{t+1}}{1-p_{t+1}} \cdot \binom{0}{1} }_{\text{round } t+1} = (p_t) + (1-p_{t+1}).
			\end{align*}
			Similarly, the loss incured by Decreasing Hedge for loss sequence $\lossv'_1, \ldots, \lossv'_T$ at rounds $t$ and $t+1$ can be written as
			\begin{align*}
				\hedgeloss'_{t} + \hedgeloss'_{t+1} &= \underbrace{\binom{p'_t}{1-p'_{t}} \cdot \binom{0}{1} }_{\text{round } t} + \underbrace{\binom{p'_{t+1}}{1-p'_{t+1}} \cdot \binom{1}{0} }_{\text{round } t+1} = (1-p'_{t}) + (p'_{t+1}).
			\end{align*}
			We know that Decreasing Hedge incurs the same loss at any round $s \in [T]\setminus\{t,t+1\}$; therefore, 
			\begin{align*}
				\hat{L'}_T - \hat{L}_T &=\sum_{s=1}^T \hedgeloss'_s - \sum_{s=1}^T \hedgeloss_s \\
				&=\big( \hedgeloss'_{t} + \hedgeloss'_{t+1}\big) - \big(\hedgeloss_{t} + \hedgeloss_{t+1} \big) \\
				&= (p'_{t+1} + p_{t+1}) - (p'_t + p_t) \\
				&= \big(f(\eta_{t+1}\Delta'_{t}) + f(\eta_{t+1}\Delta_{t})\big)
				- \big(f(\eta_{t}\Delta'_{t-1})+ f(\eta_{t}\Delta_{t-1})\big)\\
				&= \big(f(\eta_{t+1}\Delta'_{t}) + f(\eta_{t+1}\Delta_{t})\big)
				- \big(2\cdot f(\eta_{t}\Delta_{t-1})\big)\\
				&= \Big(f\big(\eta_{t+1}(\Delta_{t-1}+1)\big) + f\big(\eta_{t+1}(\Delta_{t-1}-1)\big)\Big)
				- 2\cdot f(\eta_{t}\Delta_{t-1})\\
				&=2 \cdot \Big(\frac{1}{2}f\big(\eta_{t+1}(\Delta_{t-1}+1)\big) + \frac{1}{2} f\big(\eta_{t+1}(\Delta_{t-1}-1)\big)\Big)
				- 2\cdot f(\eta_{t}\Delta_{t-1})\\
				&\leq 2 \cdot f \Big( \frac{1}{2} \eta_{t+1}(\Delta_{t-1}+1)  + \frac{1}{2} \eta_{t+1}(\Delta_{t-1}-1)\Big)
				- 2 \cdot f(\eta_{t}\Delta_{t-1})\\
				&= 2 \cdot f (  \eta_{t+1}\Delta_{t-1})
				- 2 \cdot f(\eta_{t}\Delta_{t-1})\\
				&= 2 \cdot \big[ f (  \eta_{t+1}\Delta_{t-1}) 
				-  f(\eta_{t}\Delta_{t-1}) \big] \\
				&\leq 0,
			\end{align*}
			where the first inequality is Jensen's inequality applied for function $f$ which is concave for nonnegative domain and $\Delta_{t-1} -1$ and $\Delta_{t-1}+1$ are nonnegative. The second inequality comes from the fact that in Decreasing Hedge, $\eta_{t+1} < \eta_{t}$ and $f$ is an increasing function.
		\end{enumerate}
	\end{proof}
\label{prf:swapping_rules_lemma_from_Nishant}

	\subsection{Regret is minimized when both experts have the same cumulative loss}
	
	\begin{proof}[Proof of Lemma \ref{lm:c_even}]
		As $T=a+b+2k-1$, and $T$ is even, $a+b$ should be odd. Therefore, $a \neq b$. Which implies that $a > b$. Thus $a - 1 \geq b$. Therefore, the best expert in both loss sequences has the same loss:
		\begin{align*}
			\min_{i}L'_{i,T}=\min_{i}L_{i,T}=b .
		\end{align*}
		Also, $\hedgeloss'_t=\hedgeloss_t$ for all $t < a$ as the loss sequences are the same for all rounds $t<a$.

		In round $t=a$, Decreasing Hedge incurs some positive loss in $\lossv_1, \ldots, \lossv_T$ while it incurs zero loss in sequence $\lossv'_1, \ldots, \lossv'_T$. In other words, $\hedgeloss'_a = 0 < \hedgeloss_a$.

		For rounds $t \in \{a+1, \ldots, a+2k-1\}$, 	Decreasing Hedge incurs zero loss for both loss sequences $\lossv_1, \ldots, \lossv_T$ and $\lossv'_1, \ldots, \lossv'_T$. In other words, $\hedgeloss'_t = \hedgeloss_t =0$. 
		
		Since $\Delta'_{t-1} = L'_{2,t-1} - L'_{1,t-1} = (L_{2,t-1}-1) - L_{1,t-1}= \Delta_{t-1} - 1$ for any round $t \in \{a+(2k-1)+1,\ldots, a+(2k-1)+b\}$, 
		\begin{align*}
			\hedgeloss'_t =\frac{1}{1+\exp(- \eta_t \Delta'_{t-1})} =  \frac{1}{1+\exp\big(- \eta_t (\Delta_{t-1}-1)\big)} \leq \frac{1}{1+\exp(- \eta_t \Delta_{t-1})} = \hedgeloss_t.
		\end{align*}
		
		As a result, $\hedgeloss'_t \leq \hedgeloss_t$ for all $1\leq t \leq T$. Therefore, 
		\begin{align*}
			\regret'(T)=\sum_{t=1}^T{\hedgeloss'_t} - \min_i L'_{i,T} &= \sum_{t=1}^T{\hedgeloss'_t} - \min_i L_{i,T} \\
			&\leq \sum_{t=1}^T{\hedgeloss_t} - \min_i L_{i,T} = \regret(T).
		\end{align*}
	\end{proof}
\label{prf:c_even}

	\begin{proof}[Proof of Lemma \ref{lm:equalize_both_experts_cumulative_loss}]
		Consider a sequence $\lossv_1, \ldots, \lossv_T=\left({\binom{0}{1}}^{a} {\binom{0}{0}}^c {\binom{1}{0}}^{b}\right)$ where $a \neq b$. Therefore, $a > b$ which means that $a\geq b+1$. If $a=b+1$, then $a+b=2b+1\neq 2K$. This means that $a\neq b+1$. Therefore, $a \geq b+2$. 
		
		Now, consider $\lossv'_1, \ldots, \lossv'_T=\left({\binom{0}{1}}^{a-1} \binom{1}{0} {\binom{0}{0}}^c {\binom{1}{0}}^{b}\right)$. We know that the best expert in $\lossv_1, \ldots, \lossv_T$ incurs $L_{1,T}=b < a = L_{2,T}$ whereas the best expert in $\lossv'_1, \ldots, \lossv'_T$ incurs additional unit loss, since $L'_{1,T} = b+1\leq a-1 = L'_{2,T}$. Therefore, $\min_i{L'_{i,T}}=b+1 = \min_i{L_{i,T}}+1$.

		We now show that the cumulative loss of Decreasing Hedge for loss sequence $\lossv'_1, \ldots, \lossv'_T$ is at most 1 unit greater than cumulative loss for loss sequence $\lossv_1, \ldots, \lossv_T$. Therefore, $\regret'(T) \leq \regret(T)$.
		
		For rounds $t<a$, Decreasing Hedge incurs the same loss for both loss sequences, i.e., $\hedgeloss'_t=\hedgeloss_t$ for all $t < a$.
		In round $t=a$, we have $\hedgeloss'_a=\frac{1}{1+\exp(-\eta_a(a-1))}\geq 0$ and $\hedgeloss_a=\frac{1}{1+\exp(\eta_a(a-1))}\geq 0$. Observe that  $\hedgeloss'_a + \hedgeloss_a = \frac{1}{1+\exp(-\eta_a(a-1))}+\frac{1}{1+\exp(\eta_a(a-1))}=1 $. Therefore,  
		\begin{align*}
			\hedgeloss'_a - \hedgeloss_a = (1-\hedgeloss_a) - \hedgeloss_a = 1 - 2\hedgeloss_a \leq 1.
		\end{align*}

		As a result, in round $a$ Decreasing Hedge incurs at most 1 unit of additional loss in loss sequence $\lossv'_1, \ldots, \lossv'_T$ as compared to loss sequence $\lossv_1, \ldots, \lossv_T$.

		For rounds $t \in \{a+1, \ldots, a+c \}$, Decreasing Hedge incurs zero loss in both loss sequences.
		
		Note that since there is a change from $\binom{0}{1}$ to $\binom{1}{0}$ in round $a$, for any round $t \in \{ a+1,\ldots, a+c+b\}$, we have $\Delta'_{t-1} = L'_{2,t-1} - L'_{1,t-1} = \Delta_{t-1} - 2$. Therefore, for any round $t$ in this period, Decreasing Hedge incurs less loss in $\lossv'_1, \ldots, \lossv'_T$ than $\lossv_1, \ldots, \lossv_T$:
		\begin{align*}
			\hedgeloss'_t =\frac{1}{1+\exp(- \eta_t \Delta'_{t-1})} =  \frac{1}{1+\exp\big(- \eta_t (\Delta_{t-1}-2)\big)} \leq \frac{1}{1+\exp(- \eta_t \Delta_{t-1})} = \hedgeloss_t.
		\end{align*}
		As a result, the cumulative loss of Decreasing Hedge in loss sequence $\lossv'_1, \ldots, \lossv'_T$ is at most 1 unit greater that Decreasing Hedge's cumulative loss in loss sequence $\lossv_1, \ldots, \lossv_T$, i.e.,
		\begin{align*}
			\sum_{t=1}^T{\hedgeloss'_t} \leq  1 + \sum_{t=1}^T \hedgeloss_t.
		\end{align*}
		Therefore, 
		\begin{align*}
			\regret'(T) &=\sum_{t=1}^T \hedgeloss'_t - \min_i{L'_{i,T}} \\
			&=\sum_{t=1}^T \hedgeloss'_t - (\min_i{L_{i,T}}+1)\\
			&\leq(1 + \sum_{t=1}^T \hedgeloss_t) - (\min_i{L_{i,T}}+1)\\
			&=\sum_{t=1}^T \hedgeloss_t - \min_i{L_{i,T}} =  \regret(T).
		\end{align*}
		Now by Lemma \ref{lm:swapping_rules_lemma_from_Nishant}, we can move the loss vector $(1,0)$ in $\lossv'_1, \ldots, \lossv'_t$ in round $a$ to the right without increasing the regret. Therefore, the regret for 
		$\lossv''_1, \ldots, \lossv''_t = \left({\binom{0}{1}}^{a-1} {\binom{0}{0}}^c {\binom{1}{0}}^{b+1}\right)$ is no greater than the regret for $\lossv_1, \ldots, \lossv_T$. This means that when $a\geq b+1$, converting loss sequence $\left({\binom{0}{1}}^{a} {\binom{0}{0}}^c {\binom{1}{0}}^{b}\right)$ to $\left({\binom{0}{1}}^{a-1} {\binom{0}{0}}^c {\binom{1}{0}}^{b+1}\right)$ does not increase the regret. We can apply this rule many times to equalize the numbers $a$ and $b$. As $a+b=2K$, this implies $a=b=K$ and the lemma follows.
	\end{proof}
\label{prf:equalize_both_experts_cumulative_loss}

\subsection{The optimum number of $(0,0)$ loss vectors in the middle}

	\begin{proof}[Proof of Lemma \ref{lm:remove_zeros}]
	First of all, observe that both experts incur one additional unit loss in loss sequence $\lossv'_1,\ldots,\lossv'_T$ in comparison to $\lossv_1,\ldots,\lossv_T$. Therefore,
	\begin{align*}
		\max_i{L'_{i,T}}=\max_i{L_{i,T}}+1.
	\end{align*}
	Note that in any round $s$ except rounds $t=a+1$ and $t'=a+c$, we have $\hedgeloss_s = \hedgeloss'_s$.

	On the other hand, for loss $\lossv'_1,\ldots,\lossv'_T$, in round $t=a+1$ and $t'=a+c$, Decreasing Hedge incurs loss
	\begin{align*}
		\hedgeloss'_{a+1}= & \ \ \frac{1}{1+\exp(-\eta_t(a))}, \\
		\hedgeloss'_{a+c}= & \ \ \frac{1}{1+\exp(\eta_{t'}(a+1))} .
	\end{align*}
	while for $\lossv_1,\ldots,\lossv_T$, the algorithm does not incur any loss in those rounds (i.e.~$\hedgeloss_{a+1}+\hedgeloss_{a+c}=0$). Therefore, the difference in regret would be
	\begin{align*}
		\regret'(T)-\regret(T) = \frac{1}{1 + \exp(\eta_t (a))} + \frac{1}{1 + \exp(-\eta_{t'}(a+1))} - 1 .
	\end{align*}

	Let $f(x):=\frac{1}{1+\exp(-x)}$. Observe that $f(x)+f(-x)=1$. Therefore,
	\begin{align*}
		\regret'(T)-\regret(T) &= \frac{1}{1 + \exp(\eta_t (a))} + \frac{1}{1 + \exp(-\eta_{t'}(a+1))} - 1 \\
		&= f(-\eta_t(a)) + f(\eta_{t'}(a+1)) - 1 \\
		&= f\big(-\eta_t(a)\big) + f\big(\eta_{t'}(a+1)\big) - \Big(f\big(\eta_t(a)\big)+f\big(-\eta_t(a)\big)\Big) \\
		& = f(\eta_{t'}(a+1)) - f(\eta_{t}(a)).\\
	\end{align*}

	As $f$ is an increasing function, $f(\eta_{t'}(a+1)) - f(\eta_{t}(a)) \leq 0$ if and only if
	\begin{equation}
		\eta_{t'}(a+1) \leq  \eta_{t}(a) .
		\label{eq:decreasing_c_inequality}
	\end{equation}

	Therefore, if inequality \eqref{eq:decreasing_c_inequality} holds, then the replacement does not increase the regret. Substituting the learning rate, inequality \eqref{eq:decreasing_c_inequality} becomes
	\begin{equation*}
		\frac{a+1}{\sqrt{a+c}}	 \leq \frac{a}{\sqrt{a+1}},
	\end{equation*}
	which is equivalent to 
	\begin{align}
		c \geq \frac{3a^2+3a+1}{a^2}= 3 + \frac{3}{a} + \frac{1}{a^2}.
		\label{eq:inequality_for_removing_two_zeros}
	\end{align}
	
	Note that in case $a\geq 4$ and $c \geq 4$, the inequlity $\frac{3}{a} + \frac{1}{a^2} <1$  holds; therefore, \eqref{eq:inequality_for_removing_two_zeros} holds which means that the regret does not increase. 
	
	In the case $c\leq 3$, then \eqref{eq:inequality_for_removing_two_zeros} never holds which means that this replacement increases the regret. 
\end{proof}
\label{prf:remove_zeros}
	
	\subsection{Bounding the regret}
	\label{sc:calculating_the_regret}
	So far, we have shown that among all loss sequence of size $T=2K+2\geq 16$, the loss sequence $\left({\binom{0}{1}}^{K} {\binom{0}{0}}^2 {\binom{1}{0}}^{K}\right)$ yields the minimum regret for Decreasing Hedge. We will now give upper and lower bounds on the regret for this sequence. The algorithm's loss for this sequence is
	\begin{align*}
		\hat{L}_T = \sum_{t=1}^T\hedgeloss_t &= \underbrace{\hedgeloss_1 +\sum_{t=2}^K (\hedgeloss_{t} ) }_{{{\binom{0}{1}}^K}}+  \underbrace{\hedgeloss_{K+1} + \hedgeloss_{K+2}}_{{{\binom{0}{0}}^2}} + \underbrace{\hedgeloss_{K+3} + \sum_{t=2}^K \hedgeloss_{2K+4-t}}_{{\binom{1}{0}}^K} \\
		&=\hedgeloss_1 +  \hedgeloss_{K+3} + \sum_{t=2}^K \hedgeloss_{t} + \hedgeloss_{2K+4-t} \\
		&= \frac{1}{2} +  \frac{1}{1+\exp(-\eta_{K+3}(K))} + \sum_{t=2}^K 
		\Big( \frac{1}{1+\exp{\big(\eta_t(t-1)\big)}} + \frac{1}{1+\exp{\big(-\eta_{2K+4-t}(t-1)}\big)} \Big).
	\end{align*}
	
	Therefore, 
	\begin{align*}
		\regret(T) &= \hat{L}_T - \min_i L_{i,T} \\
		&= \hat{L}_T - K\\
		&= \sum_{t=2}^K 
		\Big( \frac{1}{1+\exp\big({\eta_t(t-1)\big)}} + \frac{1}{1+\exp{\big(-\eta_{2K+4-t}(t-1)}\big)} \Big) \\
		&\quad+ \frac{1}{2} + \frac{1}{1+\exp(-\eta_{K+3}(K))} - K
		\\
		&=\sum_{t=2}^K 
		\Big( \frac{1}{1+\exp\big({\eta_t(t-1)\big)}} + \frac{1}{1+\exp{\big(-\eta_{2K+4-t}(t-1)}\big)} - 1\Big) + \sum_{t=2}^K 1 \\
		&\quad+ \frac{1}{2}+\frac{1}{1+\exp(-\eta_{K+3}(K))} - K \\ 
		&= \sum_{t=2}^K\Big( \frac{1}{1+\exp\big({\eta_t(t-1)\big)}} + \frac{1}{1+\exp{\big(-\eta_{2K+4-t}(t-1)}\big)} - 1\Big)\\
		&\quad+ \frac{1}{2}+\frac{1}{1+\exp(-\eta_{K+3}(K))} - 1.
		\label{eq:regret_calculating_1}
	\end{align*}	
	As we know that $\frac{1}{1+\exp{\big(-\eta_{2K+4-t}(t-1)}\big)} + \frac{1}{1+\exp{\big(+\eta_{2K+4-t}(t-1)}\big)} = 1$,
        it follows that
	\begin{align*}
		\frac{1}{1+\exp\big({\eta_t(t-1)\big)}} + \frac{1}{1+\exp{\big(-\eta_{2K+4-t}(t-1)}\big)} - 1 = \frac{1}{1+\exp\big({\eta_t(t-1)\big)}} - \frac{1}{1+\exp{\big(\eta_{2K+4-t}(t-1)}\big)}.
	\end{align*} 
Hence, we can rewrite the regret as 	
	\begin{align*}
		\regret(T) &= 
		\underbrace{\sum_{t=2}^K \frac{1}{1+\exp\big({\eta_t(t-1)\big)}}}_{\text{A = First term}} - \underbrace{\sum_{t=2}^K \frac{1}{1+\exp{\big(\eta_{2K+4-t}(t-1)}\big)}}_{\text{B = Second term}}+\underbrace{\frac{1}{2}+\frac{1}{1+\exp(-\eta_{K+3}(K))} - 1}_{\text{C =  Third term}}.
	\end{align*}
	
	We will give upper and lower bounds for each term separately. Note that $K=\frac{T-2}{2}$. In the development of the bounds below, we assume that $T$ is sufficiently large. It suffices to have $T \geq 16$.

	Now, for the A (First term), we have
	\begin{align*}
		0 \leq \text{A} &= \sum_{t=2}^K \frac{1}{1+\exp(\eta_t(t-1))}=
		\sum_{t=2}^K \frac{1}{1+\exp(\sqrt{t} - \frac{1}{\sqrt{t}})}\\
		&\leq \sum_{t=2}^K \frac{1}{1+\exp(\frac{1}{2} \sqrt{t})} \leq  \sum_{t=2}^K \frac{1}{\exp(\frac{1}{2} \sqrt{t})}\\
		&\leq \int_1^K \frac{1}{\exp(\frac{1}{2} \sqrt{x})} \,dx \\
		&=(-4\sqrt{K}-8)\exp(-0.5 \sqrt{K}) - (-4\sqrt{1}-8)\exp(-0.5 \sqrt{1})\\
		&\leq (4\sqrt{1}+8)\exp(-0.5 \sqrt{1})\\
		& = \frac{12}{\sqrt{e}}.
	\end{align*}
	To upper bound B (Second term), we divide the $T$ rounds into $\left\lceil \sqrt{T}\right\rceil$ intervals. We then show that sum of the contributions of the intervals is upper bounded by sum of a geometric series that converge to a constant value. 
	
	The second term can be written as the sum of  
$\left\lceil\sqrt{T}\right\rceil$ intervals as  
	\begin{align*}
		\sum_{t=2}^K \frac{1}{1+\exp{\big(\eta_{2K+4-t}(t-1)}\big)} 
		&=\sum_{t=2}^{\frac{T}{2}-1} \frac{1}{1+\exp{\big(\eta_{T+2-t}(t-1)}\big)} \\
		&\leq  \sum_{t=1}^{T} \frac{1}{1+\exp{\big(\frac{t-1}{\sqrt{T+2-t}}}\big)}  \\
		&= \sum_{k=0}^{\left\lceil\sqrt{T}\right\rceil-1} \,\, \sum_{t=k \left\lceil\sqrt{T}\right\rceil+1}^{k\left\lceil\sqrt{T}\right\rceil+\left\lceil\sqrt{T}\right\rceil}\frac{1}{1+\exp{\big(\frac{t-1}{\sqrt{T+2-t}}}\big)}\\
		&\leq \sum_{k=0}^{\left\lceil\sqrt{T}\right\rceil-1} \,\, \sum_{t=k \left\lceil\sqrt{T}\right\rceil+1}^{k\left\lceil\sqrt{T}\right\rceil+\left\lceil\sqrt{T}\right\rceil} \frac{1}{1+\exp({\frac{k\sqrt{T}}{\sqrt{T+2}}})}\\
		& \leq \left\lceil\sqrt{T}\right\rceil \sum_{k=0}^{\left\lceil\sqrt{T}\right\rceil - 1} \frac{1}{1+\exp({\frac{k\sqrt{T}}{2+\sqrt{T}}})}.
	\end{align*}
	The term
$\displaystyle \sum_{k=0}^{\left\lceil\sqrt{T}\right\rceil-1} \frac{1}{1+\exp({\frac{k\sqrt{T}}{2+\sqrt{T}}})}$ 
 can be upper bounded as
	\begin{align*}
		\sum_{k=0}^{\left\lceil\sqrt{T}\right\rceil-1} \frac{1}{1+\exp({\frac{k\sqrt{T}}{2+\sqrt{T}}})}
		=  \sum_{k=0}^{\left\lceil\sqrt{T}\right\rceil-1} \frac{1}{1+\exp\big({k - \frac{2k}{2+ \sqrt{T}}}\big)}
		& \leq \sum_{k=0}^{\left\lceil\sqrt{T}\right\rceil-1} \frac{1}{1+\exp(k-2)}\\
		&\leq \sum_{k=0}^{\left\lceil\sqrt{T}\right\rceil-1} \frac{1}{\exp(k-2)} \\
		&= { e^2 } \sum_{k=0}^{\left\lceil\sqrt{T}\right\rceil-1}  {\left(\frac{1}{e}\right)}^{k} \\
		&\leq { e^2 } \left( \frac{1-{\frac{1}{e}}^{\left\lceil\sqrt{T}\right\rceil}}{1-\frac{1}{e}} \right) \\
		&\leq \frac{ {e^2 }}{1-\frac{1}{e}}.
	\end{align*}
As a result, 
\begin{align*}
	\sum_{t=2}^K \frac{1}{1+\exp{\big(\eta_{2K+4-t}(t-1)}\big)} \leq \left\lceil\sqrt{T}\right\rceil \cdot \sum_{k=0}^{\left\lceil\sqrt{T}\right\rceil - 1} \frac{1}{1+\exp({\frac{k\sqrt{T}}{2+\sqrt{T}}})} \leq \frac{ {e^2 }}{1-\frac{1}{e}} \cdot \left\lceil \sqrt{T} \right\rceil .
\end{align*}

	Moreover,  B (Second term) can be lower bounded as follows:
	\begin{align*}
		\sum_{t=2}^K \frac{1}{1+\exp{\big(\eta_{2K+4-t}(t-1)}\big)} 
                &= \sum_{t=2}^K \frac{1}{1+\exp{\big(\eta_{T+2-t}(t-1)}\big)}\\
		&= \sum_{t=2}^{\frac{T}{2}-1} \frac{1}{1+\exp(\frac{t-1}{\sqrt{T+2-t}})}\\
		&\geq \sum_{t=2}^{\left\lceil\sqrt{T}\right\rceil+1} \frac{1}{1+\exp(\frac{t-1}{\sqrt{T+2-t}})}\\
		&\geq \sum_{t=2}^{\left\lceil\sqrt{T}\right\rceil+1} \frac{1}{1+\exp\left({\frac{\left\lceil\sqrt{T}\right\rceil}{\sqrt{T+2-t}}}\right)}\\
		&\geq \sum_{t=2}^{\left\lceil\sqrt{T}\right\rceil+1} \frac{1}{1+\exp\left({\frac{\sqrt{T}}{\sqrt{T+1-\left\lceil\sqrt{T}\right\rceil}}}\right)}\\
		&\geq \sum_{t=2}^{\left\lceil\sqrt{T}\right\rceil+1} \frac{1}{1+\exp\left({\frac{\left\lceil\sqrt{T}\right\rceil}{\sqrt{\frac{1}{2}}\left\lceil\sqrt{T}\right\rceil}}\right)}\\
		&= \sum_{t=2}^{\left\lceil\sqrt{T}\right\rceil+1} \frac{1}{1+\exp({\sqrt{2}})}\\
		&\geq \frac{\sqrt{T}}{1+\exp({\sqrt{2}})} .
	\end{align*}
	
	Therefore, 
	\begin{align*}
		\frac{1}{1+e^{\sqrt{2}}}\sqrt{T} \leq B =  \sum_{t=2}^K \frac{1}{1+\exp{\big(\eta_{2K+4-t}(t-1)}\big)} \leq \frac{e^2}{ (1-\frac{1}{e})} \left\lceil\sqrt{T}\right\rceil .
	\end{align*}

	For the C (Third term), as we know 
$0 \leq \hedgeloss_{K+3} = \frac{1}{1+\exp(-\eta_{K+3}(K))} \leq 1$, 
it follows that
	\begin{align*}
	-\frac{1}{2}\leq \text{C}  \leq \frac{1}{2}.
	\end{align*}
	
	As a result, we have
	\begin{align*}
		-\frac{e^2}{ (1-\frac{1}{e})}\cdot \left\lceil\sqrt{T}\right\rceil  -\frac{1}{2} \leq A-B+C \leq \frac{12}{\sqrt{e}} +\frac{1}{2} - \frac{1}{1+e^{\sqrt{2}}}\sqrt{T},
	\end{align*}
	which for suitably defined constants is 
	\begin{align*}
		-c_1 \sqrt{T} + c_2 \leq \regret(T) \leq c_3 - c_4\sqrt{T}.
	\end{align*}
	Therefore, 
	\begin{align*}
		\regret(T) = -\Theta(\sqrt{T}).
	\end{align*}

\end{document}